\documentclass{article}

\usepackage{etoolbox}
\makeatletter
\usepackage{ragged2e}

\usepackage{arxiv}

\usepackage[utf8]{inputenc} % allow utf-8 input
\usepackage[T1]{fontenc}    % use 8-bit T1 fonts
\usepackage{hyperref}       % hyperlinks
\usepackage{url}            % simple URL typesetting
\usepackage{booktabs}       % professional-quality tables
\usepackage{amsfonts}       % blackboard math symbols
\usepackage{nicefrac}       % compact symbols for 1/2, etc.
\usepackage{microtype}      % microtypography
\usepackage{subfigure}
\usepackage{placeins}
\usepackage{tikz}
\usetikzlibrary{shapes.geometric, arrows}% 
\usepackage{amsmath, tikz-cd}
\usepackage{amssymb}
\usepackage{graphicx}
\usepackage{xcolor}         % colors
\usepackage{mathrsfs}
\usepackage{commath}
\usepackage{accents}
\usepackage{natbib}
\usepackage{mathtools}
\usepackage[bb=dsserif]{mathalpha}
\usepackage{bm}
\usepackage[shortlabels]{enumitem}
\usepackage{xparse}
\usepackage{dsfont}
\usepackage{bbm}
\usepackage{caption}
\usepackage{subcaption}
\DeclareMathOperator*{\argmin}{arg\,min}
\usepackage{float}
\usepackage{threeparttable}
\usepackage{booktabs}
\usepackage{multirow}
\usepackage{pifont}
\usetikzlibrary{calc}
\usepackage{algorithm, algorithmic}
\usepackage{wrapfig}
\usepackage{fontawesome}
\usepackage[dvipsnames]{xcolor}
\hypersetup{colorlinks=true,citecolor=Aquamarine, linkcolor=magenta}

\usepackage{amsthm}
\usepackage{subcaption}
\usepackage[capitalize,noabbrev]{cleveref}

\theoremstyle{plain}

\newtheorem{theorem}{Theorem}[section]

\newtheorem{lemma}[theorem]{Lemma}

\theoremstyle{definition}

\theoremstyle{remark}
\newtheorem{remark}[theorem]{Remark}

\usepackage[textsize=tiny]{todonotes}

\title{Optimal Anytime Algorithms for Online Convex Optimization with Adversarial Constraints}

\author{
 Dhruv Sarkar
 \\Department of Computer Science and Engineering,\\
  Indian Institute of Technology, Kharagpur \\
  \And
  Abhishek Sinha
  \\School of Technology and Computer Science,
  \\Tata Institute of Fundamental Research, Mumbai
}

\begin{document}
\maketitle
\begin{abstract}
      We propose an anytime online algorithm for the problem of learning a sequence of adversarial convex cost functions while approximately satisfying another sequence of adversarial online convex constraints. A sequential algorithm is called \emph{anytime} if it provides a non-trivial performance guarantee for any intermediate timestep $t$ without requiring prior knowledge of the length of the entire time horizon $T$. Our proposed algorithm achieves optimal performance bounds without resorting to the standard doubling trick, which has poor practical performance due to multiple restarts. Our core technical contribution is the use of time-varying Lyapunov functions to keep track of constraint violations. This must be contrasted with prior works that used a fixed Lyapunov function tuned to the known horizon length $T$. The use of time-varying Lyapunov function poses unique analytical challenges as properties, such as \emph{monotonicity}, on which the prior proofs rest, no longer hold. By introducing a new analytical technique, we show that our algorithm achieves $O(\sqrt{t})$ regret and $\tilde{O}(\sqrt{t})$ cumulative constraint violation bounds for any $t\geq 1$.

%To show the generality of our methodology, we extend our approach to other settings where knowledge of certain parameters is often presupposed or estimated using the doubling trick. First, we develop an adaptive algorithm for the 
We extend our results to the dynamic regret setting, achieving bounds that adapt to the path length of the comparator sequence without prior knowledge of its total length. We also present an adaptive algorithm in the optimistic setting, whose performance gracefully scales with the cumulative prediction error. We demonstrate the practical utility of our algorithm through numerical experiments involving the online shortest path problem.
\end{abstract}
\section{Introduction}
%We consider the problem of online learning of adversarially generated convex cost and constraint functions. This problem setting is well-known and is often referred to as Constrained Online Convex Optimization (COCO for short) to highlight that it is a generalization of the Online Convex Optimization paradigm. 
We consider a natural generalization of the standard Online Convex Optimization (OCO) problem, called \underline{C}onstrained \underline{O}nline \underline{C}onvex \underline{O}ptimization (COCO), where in addition to a convex cost function, a convex constraint function is revealed to the learner after it chooses its action at each round for a horizon of $T$ rounds. 
The goal is to design an online algorithm which achieves 
 a small regret (w.r.t. the cost functions) while ensuring a small cumulative constraint violation or CCV (w.r.t. the constraint functions) for all rounds $t \geq 1$. %against any adaptive adversary that is allowed to choose a single action across all rounds. 
 The COCO problem arises in many practical settings, including safety-aware contextual bandits \citep{pmlr-v70-sun17a}, autonomous driving \citep{gao2024constraints}, budget-constrained bandits \citep{immorlica2022adversarial}, and learning safety constraints in LLMs \citep{chen2025learning}.

The COCO problem has been extensively studied in the literature over the past decade \citep{guo2022online, yi2023distributed, neely2017online}. Recent work by \citet{sinha2024optimal} introduced a simple algorithm which yields $O(\sqrt{T})$ regret and $\Tilde{O}(\sqrt{T})$ CCV guarantees, which are also shown to be tight when the horizon length $T$ is known. These results are established by leveraging a Lyapunov function-based technique that constructs a surrogate cost function by linearly combining the cost and constraint functions at each round, and then passes it on to a standard OCO algorithm. The analysis is also refreshingly elegant in stark contrast to prior works that rely on primal-dual-based approaches and achieve only suboptimal guarantees \citep{guo2022online, yi2021regret, yi2023distributed, pmlr-v70-sun17a}.

However, a major drawback of their result is that their proposed algorithm is \emph{not} anytime as it needs to know the length of the horizon $T$ \emph{a priori}. In particular, although their algorithm achieves $O(\sqrt{t})$ regret for any $t \geq 1,$ the CCV guaranteed by their algorithm is not time-uniform as it only achieves a CCV of $\Tilde{O}(\sqrt{T})$ for any time $t \geq 1$, however small \citep[Theorem 1]{sinha2024optimal}. This raises a natural and important question - is it possible to design a COCO algorithm which simultaneously achieves $O(\sqrt{t})$ regret and $\tilde{O}(\sqrt{t})$ CCV for any $t \geq 1?$
%This was an artifact of the fact that they used the same Lyapunov function $\Phi$ across the timesteps in their algorithm. 
%Further, their algorithm required knowledge of the time horizon $T$ apriori which may not be acceptable in many practical scenarios.

In this paper, we affirmatively answer this question by introducing a sequence of time-varying Lyapunov functions $\{\Phi_t(\cdot)\}_{t \geq 1},$ which replace the fixed, horizon-dependent Lyapunov function $\Phi(\cdot)$ employed by \citet{sinha2024optimal}. Their analysis crucially utilized the monotone \emph{non-decreasing} property of the $\{\Phi(Q(t))\}_{t \geq 1}$ sequence, where $Q(t)$ is the CCV up to round $t$. Unfortunately, it turns out that with the new Lyapunov functions, $\Phi_t(x)$ is actually \emph{monotonically decreasing} pointwise in $t$, and hence, the previous arguments do not apply. The key technical novelty in our analysis is in ensuring that the sequence $\{\Phi_t(Q(t))\}_{t\geq 1}$ remains monotonically non-decreasing even though $\Phi_t(x) < \Phi_{t-1}(x), \forall x$. This is achieved by defining $Q(t)$ to be a suitable upper bound to the current CCV while adequately exploiting the exponential form of the Lyapunov functions (see Eqn.\ \eqref{new_q_recursion}). 
%While prior works like \citet{sinha2024optimal} used simple additive updates in the recursion, we introduced a multiplicative factor to it to compensate for the progressively smaller Lyapunov functions.  

In summary, we make the following contributions in the paper:
\begin{itemize}
    \item We propose a simple anytime algorithm for COCO that achieves $O(\sqrt{t})$ regret and $\tilde{O}(\sqrt{t})$ CCV for any $t \geq 1.$ Our algorithm does not need multiple restarts and completely avoids the impractical doubling trick (see section \ref{sec:doubling_trick} for more details).
    %At a conceptual level, we demonstrate that anytime guarantees can be obtained in this framework without depending on relatively less practical methods like the doubling trick that depend on restarts. Using the doubling trick implies that the algorithm has to be restarted multiple times and may lose the progress it has made so far besides also causing instability in predictions (see section \ref{sec:doubling_trick} for more details). Thus, we are able to bound the violation up to an intermediate time $t$ in terms of $t$ instead of the final horizon length $T$, using an adaptive methodology. This means our bound is tighter than those with only terminal guarantees. 

    \item The above results are obtained by introducing a sequence of time-varying Lyapunov functions in both our algorithms and their analysis. Compared to the prior works, which use time-invariant Lyapunov functions, time-varying Lyapunov functions pose non-trivial technical challenges, as the key monotonicity property is lost. We address this challenge by working with a suitable upper bound to CCV and appropriately tuning the parameters of the Lyapunov functions. 
    %The technical challenges associated with this new instantiation of Lyapunov stability are managed by altering the queuing dynamics of the queue $Q$, which is used to track constraint violation. We do so by introducing a multiplicative factor in the queuing recursion alongside the usual additive updates.

    \item We show that our techniques extend naturally to the dynamic regret and optimistic settings. In particular, we design algorithms with performance bounds that adapt to unknown parameters, such as the total path length $\mathcal{P}_T$ or the total prediction error $\mathcal{E}_T,$ without resorting to the practically unwieldy doubling trick.
    
    \item Finally, we demonstrate the superiority of our algorithm through a number of numerical experiments.
    %We show that the gains by our algorithm are not only theoretical and asymptotic but also translate to practical applications. We perform numerical simulations on semi-synthetic data that show that our algorithm outperforms prior algorithms with only terminal guarantees. 
\end{itemize}
\section{Related Works}
Online learning with long-term constraints was first studied by ~\citet{mannor2009online}. In the context of two-player infinite-horizon stochastic games, they established a fundamental impossibility result: it is not possible to simultaneously achieve sublinear bounds for both regret as well as cumulative constraint violation (CCV) against the best fixed offline action that satisfies the long-term constraint over the entire horizon.
This negative result motivated subsequent works to consider a weaker benchmark, in which the benchmark satisfies the constraints at \emph{every} round \citep{mahdavi2012trading, neely2017online, guo2022online}. The goal in this line of work is to obtain the tightest regret and CCV bounds under this assumption. For time-invariant constraints, \citet{mahdavi2012trading} were the first to use Online Gradient Descent (OGD) and mirror prox-based online policies to obtain sublinear regret and sublinear CCV guarantees. Later, \citet{castiglioni2022unifying} proposed a unified meta-algorithm that achieves $\mathcal{O}(T^{\nicefrac{3}{4}})$ bounds for both approximate regret and CCV in the non-convex setting with long-term constraints. However, they made an additional \emph{Slater's condition} (strict feasibility) assumption and the assumption that the constraint functions change no faster than $O(T^{-\nicefrac{1}{4}})$. \citet{guo2022online} studied the problem with adversarial constraints, and achieved $O(\sqrt{T})$ regret and $O(T^{\nicefrac{3}{4}})$ CCV in the general setting without assuming Slater's condition. The closest to our work is the paper by  \citet{sinha2024optimal},  which achieves optimal $O(\sqrt{T})$ regret and $\tilde{O}(\sqrt{T})$ CCV bounds by reducing the problem to the standard online convex optimization on a sequence of surrogate cost functions. However, they assumed apriori knowledge of the horizon $T$ and also their violation bound was suboptimal in the sense that even for an intermediate timestep $t$, they could only establish $\tilde{O}(\sqrt{T})$ bound as opposed to the $\tilde{O}(\sqrt{t})$ that we present. This was primarily because the definition of the surrogate costs required the knowledge of the horizon length $T$. Follow-up works, such as \citet{lekeufack2025optimisticalgorithmonlineconvex}, which generalized \citet{sinha2024optimal}'s algorithm to the optimistic and dynamic regret setting, required the knowledge of other problem-dependent parameters, such as path length or error bounds, which prevented their algorithms from being truly adaptive. All of the above papers relied on the doubling trick to estimate these problem dependent parameters. However, as noted in a number of prior works,  the doubling trick is impractical and wasteful as it relies on restarting the algorithm from scratch every time the problem dependent parameter exceeds the estimate \citep{zhang2024improving, kwon2014continuous, luo2014towards}. Consequently, the data from the previous learning phases are not utilized for prediction in the subsequent phase, which leads to unsatisfactory practical performance.
%Further, it provides limited theoretical inspiration to the community, as its use is a last resort when truly adaptive methods introducing new analytical techniques have not been found. 
We have discussed the doubling trick in more detail in Section \ref{sec:doubling_trick} in the Appendix.
\section{Problem formulation} 
We consider a repeated game between an online player and an adaptive adversary. In this game, on each round $t\geq 1,$ the player chooses a feasible action $x_t$ from a decision set $\mathcal{X}$. The set $\mathcal{X}$ is assumed to be non-empty, closed, and convex with a finite Euclidean diameter of $D.$ Upon observing the current action $x_t,$ the adversary chooses two convex and $G$-Lipschitz functions - a \emph{cost} function $f_t: \mathcal{X} \to \mathbb{R}$ and a \emph{constraint} function $g_t: \mathcal{X} \to \mathbb{R}.$ The constraint function $g_t$ corresponds to an online constraint of the form $g_t(x)\leq 0.$ Since the constraints are adaptively chosen, they cannot be expected to be satisfied by the player every round. Consequently, the player incurs a cost of $f_t(x_t)$ and an instantaneous constraint violation of $\max(0, g_t(x_t)).$ Our objective is to design an online policy that achieves a small cumulative cost while incurring a small CCV for each $t \geq 1$. 
%Throughout the paper, we assume that all cost and constraint functions are convex and $G$-Lipschitz.
%We assume that the sub-gradients of all cost and constraint functions are uniformly upper bounded by $G.$ 
The (static) regret of any policy is computed by comparing its cumulative cost against that of a fixed feasible action $x^\star \in \mathcal{X}$ that satisfies all constraints on each round. Specifically, let $\mathcal{X}^\star \subseteq \mathcal{X}$ be the set consisting of all actions satisfying all constraints:
\begin{eqnarray} \label{feas-set}
	\mathcal{X}^\star = \{x \in \mathcal{X}: g_t(x) \leq 0, ~\forall t \geq 1\}.
\end{eqnarray} 
We assume that the feasible set is non-empty, \emph{i.e.,} $\mathcal{X}^\star \neq \emptyset.$
We then define the Regret and the CCV of any policy as 
\begin{eqnarray}
	\textrm{Regret}_T &=& \sup_{x^\star \in \mathcal{X}^\star} \sum_{t=1}^T\big(f_t(x_t)-f_t(x^\star)\big), \label{reg-def}\\
	\textrm{CCV}_T&=& \sum_{t=1}^T (g_t(x_t))^+, \label{ccv-def}
\end{eqnarray}
where we have defined $(y)^+\equiv \max(0,y), ~ y \in \mathbb{R}.$ 

In the standard OCO problem, only the cost functions are revealed on every round and there is no online constraint function. The goal for the standard OCO problem is to only minimize the regret.
Hence, the standard OCO problem can be seen to be a special case of COCO where $g_t=0, \forall t,$ and hence, $\mathcal{X}^\star=\mathcal{X}$.

\section{Anytime bounds on Regret and Cumulative Constraint Violation}
\label{tighter_bds}

\subsection{Technical Overview}
Our approach builds upon the Lyapunov-based framework recently proposed by \citet{sinha2024optimal}. They reduce the COCO problem to an instance of the standard OCO problem on a sequence of surrogate cost functions that linearly combine the cost and constraint functions at each round. The coefficients of the linear combinations are chosen according to a fixed Lyapunov function $\Phi(x):= e^{\lambda x}-1,$ where the parameter $\lambda$ depends on the length of the time horizon $T.$  Hence, their algorithm is not anytime and the resulting bounds are not time-uniform. To overcome this, we use time-varying Lyapunov functions, $\Phi_t(x) = e^{\lambda_t x} - 1$, where the parameter $\lambda_t$ is a decreasing function of the current time step $t$. 

Although the time-varying Lyapunov functions obviates the need to know the horizon length $T$, their analysis presents a significant challenge: the sequence $\Phi_t(Q(t))$ is no longer guaranteed to be monotonically non-decreasing, a property that was critical to the analysis in \citet{sinha2024optimal}. Our key technical novelty is to restore this monotonicity by redefining the variables $Q(t)$. Instead of the standard additive update, which leads to an exact expression for CCV, we introduce a multiplicative factor, defining  $Q(t) = \frac{\lambda_{t-1}}{\lambda_t}Q(t-1) + \tilde{g}_t(x_t)$. Since $\lambda_t$ is decreasing, the factor $\frac{\lambda_{t-1}}{\lambda_t} \ge 1$ precisely compensates for the shrinking Lyapunov function, ensuring that $\Phi_t(Q(t)) \ge \Phi_{t-1}(Q(t-1)).$ Furthermore, the new definition also ensures that $Q(t)$ remains an upper bound to the CCV for any $t \geq 1.$

Finally, we construct a surrogate cost function $\hat{f}_t$ based on the time-varying Lyapunov functions, which are then passed to a standard adaptive OCO algorithm like AdaGrad. This procedure allows us to derive the desired anytime bound of $O(\sqrt{t})$ for regret and $\tilde{O}(\sqrt{t})$ for cumulative constraint violation. 
\subsection{Preliminaries: Lipschitz-Adaptive OCO Algorithms}
\begin{algorithm}
\caption{\textsf{AdaGrad}: Online Gradient Descent with Adaptive step sizes} \label{ogd_alg}
\begin{algorithmic}[1]
    \STATE \textbf{Input} : Convex decision set $\mathcal{X}$, sequence of convex cost functions $\{\hat{f}_t\}_{t=1}^{T}$, sequence of learning rates $\{\eta_t\}_{t=1}^{T}$, $\textrm{diam}(\mathcal{X}) = D,$  $\mathcal{P}_{\mathcal{X}}(\cdot)=$ Euclidean projection operator on the convex set $\mathcal{X}$.
    \STATE \textbf{Initialize} : $x_1 \in \mathcal{X}$ arbitrarily.
   % \State \textbf{Parameters}: Adaptive step size sequence \[\eta_t = \frac{\sqrt{2} D}{2 \sqrt{\sum_{\tau=1}^t ||\nabla_\tau||_2^2}}, t\geq 1\] 
    %\State \textbf{Intitialization}: Set $x_1 \in \mathcal{X}, Q(0)=0$
    \FOR{$t=1:T$}
        \STATE Play $x_t$ and compute $\nabla_t = \nabla \hat{f}_t(x_t)$
        \STATE Set $x_{t+1} = \mathcal{P}_{\mathcal{X}}(x_t - \eta_t \nabla_t)$ 
    \ENDFOR
    \end{algorithmic}
\end{algorithm}
In this Section, we briefly recall the family of first-order methods (\emph{a.k.a.} Projected Online Gradient Descent (OGD)) for the standard OCO problem, which will be used as a subroutine in our proposed COCO algorithm. These methods differ from each other in the way the step sizes are chosen. For a sequence of convex cost functions $\{\hat{f}_t\}_{t \geq 1},$ a projected OGD algorithm selects the successive actions as \citep[Algorithm 2.1]{orabona2019modern}:
\begin{eqnarray}\label{ogd-policy} 
	x_{t+1} = \mathcal{P}_\mathcal{X}(x_t - \eta_t \nabla_t), ~~ \forall t\geq 1,
\end{eqnarray}
where $\nabla_t \equiv \nabla \hat{f}_t(x_t)$ is a subgradient of the function $\hat{f}_t$ at $x_t$, $\mathcal{P}_\mathcal{X}(\cdot)$ is the Euclidean projection operator on the set $\mathcal{X}$ and $\{\eta_t\}_{t \geq 1}$ is an adaptive step size sequence. 
The (diagonal version of the) \textsf{AdaGrad} policy adaptively chooses the step size sequence as a function of the norm of the previous subgradients as  $\eta_t= \frac{\sqrt{2}D}{2\sqrt{\sum_{\tau=1}^{t} G_\tau^2}},$ where $G_t=||\nabla_t||_2, t \geq 1$ \citep{duchi2011adaptive}. \footnote{We set $\eta_t=0$ if $G_t=0.$} Importantly, the \textsf{AdaGrad} policy does not need to know the horizon length $T$ or a uniform upper bound to the Lipschitz constants of the cost functions. This policy enjoys the following adaptive regret bound.
\begin{theorem}{\citep[Theorem 4.14]{orabona2019modern}}  The AdaGrad policy, with the given step size sequence, achieves the following regret bound for the standard OCO problem: 
	\begin{eqnarray} \label{cvx-reg-bd}
			 \textrm{Regret}_T \leq \sqrt{2}D \sqrt{\sum_{t=1}^T G_t^2}.
	\end{eqnarray}
	\end{theorem}
%	The OGD policy with the above adaptive step-size schedule is known as (a version of) the AdaGrad policy in the literature \citep{duchi2011adaptive}. 

\subsection{Design and Analysis of the COCO Algorithm}
 \label{sec:analysis}
To simplify the analysis, we pre-process the cost and constraint functions as follows.

\vspace{5pt}

%the queue-lengths evolve as in Eqn.\ \eqref{q-ev}, where 
\hrule
\vspace{5pt}
\textbf{Pre-processing:}
On every round, we clip the negative part of the constraint to zero by passing it through the standard ReLU unit. Next, we scale both the cost and constraint functions by a factor of $\alpha \equiv (2GD)^{-1}.$  More precisely, 
 we define $\tilde{f}_t \gets \alpha f_t, \tilde{g}_t \gets \alpha (g_t)^+.$ Hence, the pre-processed functions are $\alpha G=(2D)^{-1}$-Lipschitz with $\tilde{g}_t \geq 0, \forall t.$  
 \vspace{5pt}

\hrule 
\vspace{5pt}
In the following, we will derive the Regret and CCV bounds for the pre-processed functions. The corresponding bounds for the original functions are obtained upon multiplying the bounds by $\alpha^{-1}$.
\subsubsection{Surrogate cost functions:} 
%we can obtain the optimal regret and violation bound for the COCO problem with convex cost and convex constraint functions.
Let $\{\lambda_t\}_{t \geq 1}$ be a monotonically decreasing sequence which will be specified later.
We define $Q(t)$ using the following recursion:
\begin{eqnarray} \label{new_q_recursion}
Q(t)=\frac{\lambda_{t-1}}{\lambda_t}Q(t-1)+\tilde{g}_t(x_t), ~t\geq 1, 
\end{eqnarray}
with $Q(0)=0$.
Since $\frac{\lambda_{t-1}}{\lambda_t} \geq 1, $ $Q(t)$ is an upper bound to the $\text{CCV}$ up to round $t$. Thus, an upper bound to $Q(t)$ also implies an upper bound to CCV. Consequently, we proceed to control $Q(t).$
%Thus, making $Q(t)$ as small as possible can be one of our objectives. 

Towards this, let $\Phi_t: \mathbb{R}_+ \to \mathbb{R}_+$ be a sequence of non-decreasing convex Lyapunov functions such that $\Phi_0(0)=0.$ Using the convexity of the function $\Phi_t(\cdot),$ we have
%which generalizes Eqn.\ \eqref{dr-bd}:
%Also, for the sake of simplicity, we assume that the maximum magnitude of the constraint violation is upper bounded by $F=1$. 
%Since the function $ h: x \to x^n$ is convex, we have 
%\begin{eqnarray*}
%	\Phi(t)= Q^n(t) = \big(Q(t-1)+g_t(x_t)\big)^n \leq Q^{n}(t-1) + n Q^{n-1}(t) g_t(x_t). 
%	\end{eqnarray*}
%
\begin{align} \label{dr-bd-gen}
	\Phi_t(Q(t))  \leq \Phi_t(\frac{\lambda_{t-1}}{\lambda_t}Q(t-1)) + \Phi'_t(Q(t)) \tilde{g}_t(x_t). 
\end{align}
We now choose the Lyapunov function to be $\Phi_t(x) \stackrel{\textrm{(def)}}{=} e^{\lambda_t x}-1.$ Hence, we have $\Phi_t(\frac{\lambda_{t-1}}{\lambda_t}x) = \Phi_{t-1}(x), \forall x.$ Substituting this in the bound \eqref{dr-bd-gen}, the one-step change (\emph{drift}) of the potential function $\Phi_t(Q(t))$ can be upper bounded as 
\begin{eqnarray} \label{drift_ineq_new}
	\Phi_{t}(Q(t))-\Phi_{t-1}(Q(t-1)) \leq \Phi'_t(Q(t)) \tilde{g}_t(x_t). 
\end{eqnarray}
Recall that, in addition to controlling the CCV, we also want to minimize the cumulative cost (which is equivalent to minimizing the regret for the cost functions). We combine these two objectives into a single objective of minimizing a sequence of surrogate cost functions $\{\hat{f}_t\}_{t=1}^T$ obtained by taking a positive linear combination of the drift upper bound \eqref{drift_ineq_new} and the cost function $f_t$. More precisely, we define the surrogate cost function at round $t$ as
%Inspired by the stochastic \emph{drift-plus-penalty} framework of \citet{neely2010stochastic}, on round $t$, we attempt to minimize a surrogate cost function $\hat{f}_t : \mathcal{X} \to \mathbb{R},$ obtained by adding the scaled cost function $f_t$ to a functional form of the drift upper bound \eqref{drift_ineq_new} as defined below:
\begin{eqnarray} \label{surrogate_new}
	\hat{f}_t(x):= \tilde{f}_t(x)+ \Phi'_t(Q(t)) \tilde{g}_t(x), ~~ t \geq 1. 
\end{eqnarray}
 Our proposed policy for COCO, described in Algorithm \ref{coco_alg}, simply runs the \textsf{AdaGrad} algorithm on the surrogate cost function sequence  $\{\hat{f}_t\}_{t\geq 1}$, for a specific choice of the parameter sequence $\{\lambda_t\}_{t\geq 1}$ as dictated by the following analysis.
\begin{algorithm}[tb]
   \caption{Anytime Online Policy for COCO}
   \label{coco_alg}
\begin{algorithmic}[1]
   \STATE {\bfseries Input:} Sequence of convex cost functions $\{f_t\}_{t=1}^T$ and constraint functions $\{g_t\}_{t=1}^T,$ an upper bound $G$ to the norm of their (sub)-gradients, Diameter $D$ of the admissible set $\mathcal{X}$
     \STATE {\bfseries Parameters:} $ \lambda_{t}= \frac{1}{4\sqrt{t}\,\sqrt{\log t+1}\,(\log(\log t+1)+1)}, \Phi_t(x)= \exp(\lambda_t x)-1, \alpha=(2GD)^{-1}.$
     %$ \alpha=\frac{1}{2GD}, n=\max(2, \lceil \ln T \rceil), V=(n-1)^{n-1}T^{\frac{n-1}{2}}, \Phi(x)=x^n.$ 
%   \REPEAT
  \STATE {\bfseries Initialization:} Set $ x_1=0, Q(0)=0$.
   \FOR{$t=1:T$}
   \STATE Choose $x_t,$ observe $f_t, g_t,$ incur a cost of $f_t(x_t)$ and constraint violation of $(g_t(x_t))^+$
   \STATE $\tilde{f}_t \gets \alpha f_t, \tilde{g}_t \gets \alpha \max(0,g_t).$
   \STATE $Q(t)= \frac{\lambda_{t-1}}{\lambda_t}Q(t-1)+\tilde{g}_t(x_t).$
   \STATE Compute $\hat{f}_t$ as per \eqref{surrogate_new}
   \STATE Pass $\hat{f}_t$ and $\eta_t= \frac{\sqrt{2}D}{2\sqrt{\sum_{\tau=1}^{t} ||\nabla_\tau||_2^2}}$ to Algorithm \ref{ogd_alg}. 
%   \IF{$x_i > x_{i+1}$}
%   \STATE Swap $x_i$ and $x_{i+1}$
%   \STATE $noChange = false$
%   \ENDIF
   \ENDFOR
%   \UNTIL{$noChange$ is $true$}
\end{algorithmic}
\end{algorithm}
%\begin{framed}
%\paragraph{Policy for COCO:} Run the AdaGrad algorithm \ref{ogd-policy} on the sequence of surrogate cost functions $\{\hat{f}_t\}_{t\geq 1}$ given by Eqn.\ \eqref{surrogate_new}.
%\end{framed}

\subsubsection{The Regret Decomposition Inequality}\label{bounds_static}
Let $x^\star \in \mathcal{X}^\star$ be any feasible benchmark action as defined by Eqn.\ \eqref{feas-set}. Using the drift inequality from Eqn. \eqref{drift_ineq_new}, the definition of surrogate costs \eqref{surrogate_new}, and the fact that $g_\tau(x^\star)\leq 0, \forall \tau \geq 1,$ we have
%Working similarly as before, we have the following inequality
\begin{align*}
	&\Phi_{\tau}(Q(\tau))-\Phi_{\tau-1}(Q(\tau-1)) + (\tilde{f}_\tau(x_\tau)-\tilde{f}_\tau(x^\star)) \\ \leq &\hat{f}_\tau(x_\tau) - \hat{f}_\tau(x^\star), ~ \forall \tau \geq 1.
\end{align*}
Summing up the above inequalities for $1\leq \tau \leq t$, and using the fact that $\Phi_0(0)=0,$ we obtain  
\begin{eqnarray} \label{gen-reg-decomp}
	\Phi_t(Q(t)) + \textrm{Regret}_t(x^\star) \leq \textrm{Regret}_t'(x^\star), ~ \forall x^\star \in \Omega^\star,
\end{eqnarray}
where $\textrm{Regret}_t$ on the LHS and $\textrm{Regret}'_t$ on the RHS of \eqref{gen-reg-decomp} refers to the regret for learning the pre-processed cost functions $\{\tilde{f}_t\}_{t\geq 1}$ and the surrogate cost functions $\{\hat{f}_t\}_{t \geq 1}$ respectively (see Eqn.\ \eqref{reg-def}).
%\subsubsection{Convex cost and convex constraints}
From Eqn.\ \eqref{cvx-reg-bd}, the regret bound for the \textsf{AdaGrad} algorithm depends on the $\ell_2$ norms of the gradients of the input cost functions. Since we use \textsf{AdaGrad} for learning the surrogate cost functions $\{\hat{f}_t\}_{t \geq 1}$, we need to upper bound the gradients of the surrogate functions to derive the regret expression. Towards this, the $\ell_2$-norm of the gradients $G_t$ of the surrogate cost function $\hat{f}_t$ can be bounded using the triangle inequality as follows:
\begin{eqnarray} \label{grad_bd_new}
	%||\nabla \hat{f}_t(x_t)||_2
	G_t
	&\leq& ||\nabla \tilde{f}_t(x_t)||_2+ \Phi'_t(Q(t))||\nabla \tilde{g}_t(x_t)||_2 \nonumber\\
	&\leq& (2D)^{-1}\big(1+\Phi'_t(Q(t)\big).
\end{eqnarray}
where in the last step, we have used the fact that the pre-processed functions are $(2D)^{-1}$-Lipschitz. 
Hence, plugging in the adaptive regret bound \eqref{cvx-reg-bd} on the RHS of \eqref{gen-reg-decomp}, we arrive at the following regret decomposition inequality valid for any $t \geq 1: $
\begin{eqnarray} \label{gen-fn-ineq}
		\Phi_t(Q(t)) + \textrm{Regret}_t(x^\star) \leq \sqrt{t} + \sqrt{\sum_{\tau=1}^t \big(\Phi'_\tau(Q(\tau))\big)^2}.
\end{eqnarray}
%The above inequality is obtained by first upper-bounding 
In the above, we have utilized simple algebraic inequalities $(x+y)^2 \leq 2(x^2+y^2)$ and $\sqrt{a+b} \leq \sqrt{a} + \sqrt{b}, a, b\geq 0.$ Also note that, for our choice of Lyapunov function, $\Phi'_\tau(x) = \lambda_\tau e^{\lambda_\tau x} = \lambda_\tau (\Phi_\tau(x)+1).$ Hence,
\begin{eqnarray}
\label{eq:12}
\sqrt{\sum_{\tau=1}^t \big(\Phi'_\tau(Q(\tau))\big)^2} \leq \sqrt{\sum_{\tau=1}^t \lambda_{\tau}^2\big(\Phi_\tau(Q(\tau))+1\big)^2}.
\end{eqnarray}

Further note that the recursion $Q(\tau)=\frac{\lambda_{\tau-1}}{\lambda_\tau}Q(\tau-1)+\tilde{g}_\tau(x_\tau)$ implies $\lambda_\tau Q(\tau) \geq \lambda_{\tau-1} Q(\tau-1)$ as $\tilde{g}_\tau\geq 0.$ Therefore, $\Phi_\tau(Q(\tau))\geq \Phi_{\tau-1}(Q(\tau-1)),$ \emph{i.e.,} the sequence $\{\Phi_\tau(Q(\tau)\}_{\tau\geq 1}$ is non-negative and non-decreasing. Hence, upper-bounding all terms within the parentheses in the summation of the RHS of \eqref{eq:12} by the last term, we arrive at the following simplified bound
\begin{eqnarray*}
\sqrt{\sum_{\tau=1}^t \lambda_{\tau}^2\big(\Phi_\tau(Q(\tau))+1\big)^2} \leq \sqrt{\sum_{\tau=1}^t \lambda_{\tau}^2}(\Phi_t(Q(t))+1)
\end{eqnarray*}
Thus Eqn.\ \eqref{gen-fn-ineq} yields for any $t \geq 1:$
\begin{align} \label{gen-fn-ineq2}
	&\Phi_t(Q(t)) + \textrm{Regret}_t(x^\star) \nonumber \\ \leq &\sqrt{\sum_{\tau=1}^t \lambda_{\tau}^2}\Phi_t(Q(t)) + \sqrt{\sum_{\tau=1}^t \lambda_{\tau}^2}+\sqrt{t}.
\end{align}
The regret decomposition inequality \eqref{gen-fn-ineq2} constitutes the key step for the subsequent analysis. We finally choose the parameters  \[\lambda_\tau = \frac{1}{4\sqrt{\tau}\,\sqrt{\log(\tau)+1}\,(\log(\log(\tau)+1)+1)}, \tau \geq 1.\] Note that $\sum_{\tau=1}^t \lambda_\tau^2 < \frac{1}{4}$ using simple calculus as given in Lemma \ref{lemma:step-size-special}. 
This lets us conclude from \eqref{gen-fn-ineq2} that,
\begin{align}
\label{eq:15}
\text{Regret}_t(x^\star) + \frac{1}{2}\Phi_t(Q(t)) \leq \sqrt{t} + \frac{1}{2}
\end{align}
Noting that, $\Phi_t(Q(t)) \geq 0$, we get,
\begin{eqnarray*}
    \text{Regret}_t(x^\star) \leq \sqrt{t} + 1
\end{eqnarray*}
we trivially have $\textrm{Regret}_t(x^\star)\geq -\frac{Dt}{2D} \geq -\frac{t}{2}.$ Hence, from Eqn.\ \eqref{eq:15}, we have,
\begin{align*} \label{q-len-exp-bd}
	&\exp(\lambda_t Q(t)) \leq  t + 2\sqrt{t} + 1 \leq 4t \\ &\implies  Q(t) \leq  \frac{1}{\lambda_t}\log 4t \\ &\implies Q(t) \leq  4\sqrt{t}\,\sqrt{\log t+1}\,\bigl(\log(\log t+1)+1\bigr) \log 4t.
\end{align*}

Further, noted above, since $\{\lambda_t\}_{t\in[T]}$ is a decreasing sequence, $\textrm{CCV}_t \leq Q(t)$ based on the definition of the $\{Q(t)\}$ sequence. The following theorem summarizes the above results
\begin{theorem} \label{main_result}
For the COCO problem with adversarially chosen $G$-Lipschitz cost and constraint functions, Algorithm \ref{coco_alg} yields the following Regret and CCV bounds for any $t \geq 1:$
%for any $T \geq 1$
\begin{eqnarray*}
&&\textrm{Regret}_t \leq 2GD(\sqrt{t}+1),\\ &&\textrm{CCV}_t \leq 8GD\sqrt{t}\,\sqrt{\log t+1}\,\bigl(\log(\log t+1)+1\bigr)\log 4t.
 \end{eqnarray*}
 In the above, $D$ denotes the diameter of the decision set $\mathcal{X}$.
\end{theorem}
%\iffalse
\begin{remark}
    Theorem \ref{main_result} ensures a regret bound of $O(\sqrt{t})$ and a violation bound of $O(\sqrt{t}(\log t)^{3/2}\log \log t).$ With a slightly tighter analysis, one could improve the violation bound to $O(\sqrt{t}\log t~(\prod_{k=1}^{\log^\star t} \log^{(k)} t)^{1/2})$ by choosing $\lambda_{\tau}$ appropriately. Here $\log^{(k)}$ represents the logarithm repeated $k$ times while $\log^\star t$ represents the iterated logarithm.
\end{remark}
%\fi
\begin{remark}
We can derive an adaptive CCV bound of the form $\tilde{O}\big(\sqrt{\sum_{\tau = 1}^t ||\nabla g_t^+(x_t)||_2^2}\big)$  by choosing the value of $\lambda_{\tau}$ appropriately. Concretely, denoting $\gamma_t = \sum_{\tau = 1}^t ||\nabla g_t^+(x_t)||_2^2$ and  taking $\lambda_\tau = \frac{1}{4\sqrt{\gamma_\tau}\,\sqrt{\log(\gamma_\tau)+1}\,(\log(\log(\gamma_\tau)+1)+1)}$ would suffice. This bound is tighter because it can adjust itself based on the problem instance, whereas the $\tilde{O}(\sqrt{t})$ bound was a worst-case bound which remained the same for all problem instances. 
\end{remark}
\section{Extension to Dynamic Regret Bounds}

While in the previous Section, our algorithm ensured a static regret guarantee against a fixed benchmark, in this Section, we show that our algorithmic technique also yields a dynamic regret guarantee against any sequence of time-varying benchmarks. More importantly, our algorithm does not need to know any upper bound on the worst-case path length corresponding to the benchmark sequence. 
\subsection{Adaptive OCO meta-algorithm to minimize Dynamic Regret}
Let $\{\hat{f}_t\}_{t \geq 1},$ be a sequence of convex cost functions and $\{x_t^\star\}_{t \geq 1}$ any feasible comparator sequence such that $g_t(x_t^\star) \leq 0, \forall t\geq 1$. In other words, the comparators satisfy the online constraint functions at every round. Then the dynamic regret against the comparator sequence $\{x_t^\star\}_{t=1}^T$ of any online policy which produces an action sequence $\{x_t\}_{t \geq 1}$ is defined to be: 
\begin{eqnarray} \label{dyn-reg-def}
	\textrm{D-Regret}_T \equiv \sum_{t=1}^T \hat{f}_t (x_t) - \sum_{t=1}^T \hat{f}_t(x_t^\star). 
\end{eqnarray}
We define the path length $\mathcal{P}_T$ of a comparator sequence $\{x_t^\star\}_{t=1}^T$ as follows: 
\begin{eqnarray} \label{path-len-def}
 \mathcal{P}_T \equiv  \sum\limits_{t=1}^{T-1} ||x_{t+1}^\star - x_{t}^\star||_2.
\end{eqnarray}
Note that for the static regret metric defined in \eqref{reg-def}, we have $x_t^\star = x^\star, \forall t\geq 1$, and hence, the path-length is zero \citep{hazan2022introduction}.
\begin{algorithm}[tb]
   \caption{Anytime Online Policy for Dynamic COCO}
   \label{coco_alg_dynamic}
\begin{algorithmic}[1]
   \STATE {\bfseries Input:} Sequence of convex cost functions $\{f_t\}_{t=1}^T$ and constraint functions $\{g_t\}_{t=1}^T,$ an upper bound $G$ to the norm of their (sub)-gradients, Diameter $D$ of the admissible set $\mathcal{X}$
     \STATE {\bfseries Parameters:} $\alpha=(2GD)^{-1}.$
     
     %$\lambda_{t}= \frac{1}{4\sqrt{t(1+\mathcal{P}_t)}\,\sqrt{\log t+1}\,(\log(\log t+1)+1)}, \Phi_t(x)= \exp(\lambda_t x)-1, \alpha=(2GD)^{-1}.$
     %$ \alpha=\frac{1}{2GD}, n=\max(2, \lceil \ln T \rceil), V=(n-1)^{n-1}T^{\frac{n-1}{2}}, \Phi(x)=x^n.$ 
%   \REPEAT
  \STATE {\bfseries Initialization:} Set $ x_1=0, Q(0)=0$.
   \FOR{$t=1:T$}
   \STATE Choose $x_t,$ observe $f_t, g_t,$ incur a cost of $f_t(x_t)$ and constraint violation of $(g_t(x_t))^+$
   \STATE $\tilde{f}_t \gets \alpha f_t, \tilde{g}_t \gets \alpha \max(0,g_t).$
   \STATE $Q(t)= \frac{\lambda_{t-1}}{\lambda_t}Q(t-1)+\tilde{g}_t(x_t).$
   \STATE Compute $\hat{f}_t$ as per \eqref{surrogate_new}
   \STATE Pass $\hat{f}_t$ and $\eta_t= \frac{(D+1) (1+\mathcal{P}_t)}{ \sqrt{2\sum_{\tau=1}^t (1+\mathcal{P}_\tau)||\nabla_\tau||^2}}$ to Algorithm \ref{ogd_alg}. 
%   \IF{$x_i > x_{i+1}$}
%   \STATE Swap $x_i$ and $x_{i+1}$
%   \STATE $noChange = false$
%   \ENDIF
   \ENDFOR
%   \UNTIL{$noChange$ is $true$}
\end{algorithmic}
\end{algorithm}
% Let $\{\hat{f}_t\}_{t \geq 1},$ be a sequence of convex cost functions and $\{x_t^\star\}_{t \geq 1}$ any feasible comparator sequence. Then the dynamic regret (against the comparator sequence $\{x_t^\star\}_{t=1}^T$) of any online policy which produces the feasible action sequence $\{x_t\}_{t \geq 1}$ is defined to be: 
% \begin{eqnarray} \label{dyn-reg-def}
% 	\textrm{D-Regret}_T \equiv \sum_{t=1}^T \hat{f}_t (x_t) - \sum_{t=1}^T \hat{f}_t(x_t^\star). 
% \end{eqnarray}
% We define the path length $\mathcal{P}_T$ of a comparator sequence $\{x_t^\star\}_{t=1}^T$) as the Euclidean length of the path connecting the comparator action sequence, \emph{i.e.,} 
% \begin{eqnarray} \label{path-len-def}
%  \mathcal{P}_T \equiv  \sum\limits_{t=1}^{T-1} ||x_{t+1}^\star - x_{t}^\star||.
% \end{eqnarray}
% Note that for the weaker static regret metric where the comparator actions are time-invariant, the path-length is zero \citep{hazan2022introduction}.
For any round $t \geq 1,$ let $x_t^\star$ be an optimal feasible action, \emph{i.e.,} it is a solution to the following constrained convex optimization problem: 
\begin{eqnarray} \label{offline_opt}
	\min f_t(x), ~~~ \textrm{s.t.} ~~g_t(x) \leq 0, ~~ x \in \mathcal{X}. 
\end{eqnarray}
By definition, the worst-case dynamic regret upper bounds the dynamic regret for any arbitrary feasible comparator sequence. Also, the advantage of considering the \emph{worst-case} benchmarks is that at the end of each round $t,$ we can compute the benchmark $x_t^\star$ by solving the problem \ref{offline_opt}. This enables us to determine the running value of $\mathcal{P}_t$ at the end of each round $t$. 

For the time being, let us now consider the standard OCO problem without any constraints. The following theorem gives an upper bound to the dynamic regret of the adaptive OGD policy described in Algorithm \ref{ogd_alg} for the standard OCO problem.
\begin{theorem}[Dynamic Regret Bounds for \textsf{AdaGrad}] \label{dyn_reg_ogd}
	 The \textsf{AdaGrad} policy, described in Algorithm \ref{ogd_alg}, achieves the following adaptive dynamic regret bound for the standard OCO problem against any comparator sequence $u_{1:T}$ whose path length at time $t$ is known to be $\mathcal{P}_t$, using an adaptive learning rate schedule $\eta_t = \frac{(D+1) (1+\mathcal{P}_t)}{ \sqrt{2\sum_{\tau=1}^t (1+\mathcal{P}_\tau)||\nabla_\tau||^2}}, ~t\geq 1:$ 
	\begin{eqnarray} \label{d-regret-bd}
		     \textsc{D-Regret}_T(\mathcal{P}_T) \leq 
   (D+1) 
    \sqrt{2\sum\limits_{t=1}^T(1+\mathcal{P}_t) ||\nabla_t||^2}.
	\end{eqnarray}
    where $\nabla_t \equiv \nabla f_t(x_t), \forall t \geq 1$ and $\textrm{diam}(\mathcal{X})=D.$
    % and $\mathcal{P}_T \equiv  \sum\limits_{t=1}^{T-1} ||x_{t+1}^\star - x_{t}^\star||$ is the given upper bound on the path-length of the comparator sequence. 
\end{theorem}
\begin{remark}
Our proof of Theorem \ref{dyn_reg_ogd} adapts the static regret analysis of the \textsf{AdaGrad} algorithm from \citet[Theorem 4.14]{orabona2019modern} to the dynamic regret setting.  See Appendix \ref{dyn_reg_ogd_proof} for the details. Recently, \citet[Theorem 17]{supantha2025universaldynamicregretconstraint} also derived a dynamic regret bound for \textsf{AdaGrad} using similar techniques. The difference between our bound and theirs is that their bound assumes that the total path length $\mathcal{P}_T$ is known in advance, while we only require knowledge of the path length $\mathcal{P}_t$ up to time $t.$ This makes our approach more practical for settings where the full comparator sequence is not known beforehand. Further, our bound is tighter as it incorporates the time-varying path length $\mathcal{P}_t$ inside the summation, whereas their bound uses the total path length $\mathcal{P}_T$ as a multiplicative factor outside the sum. Since $\mathcal{P}_t \le \mathcal{P}_T$ for all $t$, the sum $\sum_{t=1}^{T}(1+\mathcal{P}_t)\|\nabla_t\|^2$ is smaller than $(1+\mathcal{P}_T)\sum_{t=1}^{T}\|\nabla_t\|^2$, resulting in a tighter overall guarantee. As we shall see, this modified base OCO algorithm would be pivotal for making our adaptive COCO algorithm. 
\end{remark}
\subsection{Dynamic Regret Decomposition Inequality}
We use the same surrogate loss as defined earlier in Eqn.\ \eqref{surrogate_new}. This yields
%Since $x^\star_\tau$ is feasible, we have $\tilde{g}_\tau(x^\star_\tau) \leq 0.$ 
%Hence, using Eqn.\ \eqref{potential_diff}, we can write:
\begin{eqnarray*}
&&\Phi_\tau(Q(\tau)) - \Phi_{\tau-1}(Q(\tau - 1))  + \tilde{f_\tau}(x_\tau) - \tilde{f_\tau}(x^\star_\tau) \\
&\stackrel{(a)}{\leq} &  \Phi_\tau'(Q(\tau))\big( \tilde{g}_\tau(x_\tau) - \tilde{g}_\tau(x^\star_\tau) \big)+ \tilde{f_\tau}(x_\tau) - \tilde{f_\tau}(x^\star_\tau)  \\
&\stackrel{(b)}{=}& \hat{f_\tau}(x_\tau) - \hat{f_\tau}(x^\star_\tau)
\end{eqnarray*}
where in step (a), we have used Eqn.\ \eqref{drift_ineq_new} and the feasibility of the benchmark $x^\star_\tau$ (which implies $g_\tau(x^\star_\tau) \leq 0$), and in step (b) we have used the definition of surrogate costs from Eqn.\ \eqref{surrogate_new}.  
Summing up the above inequalities for $\tau=1$ to $\tau=t,$ we have the following dynamic regret decomposition inequality 
\begin{equation} \label{eqn4}
    \Phi_t(Q(t)) + \textrm{D-Regret}_t (\mathcal{P}_t) \leq \textrm{D-Regret}'_t (\mathcal{P}_t), 
\end{equation}
where $\textrm{D-Regret}_t (\mathcal{P}_t)$ and $\textrm{D-Regret}_t' (\mathcal{P}_t)$ correspond to the Dynamic regrets up to round $t$ for the original cost functions $\{\tilde{f}_\tau\}_{\tau \geq 1}$ and surrogate cost functions $\{\hat{f}_\tau\}_{\tau \geq 1}$ respectively against a sequence of benchmarks with path length at most $\mathcal{P}_t$. Plugging in the dynamic regret expression for the surrogate costs from Theorem \ref{dyn_reg_ogd}, we conclude that:
\begin{align} \label{gen-fn-ineq-dynamic}
		&\Phi_t(Q(t)) + \textrm{D-Regret}_t (\mathcal{P}_t) \nonumber \\ &\leq \sqrt{1+\mathcal{P}_t}\sqrt{t} + \sqrt{\sum_{\tau=1}^t (1+\mathcal{P}_\tau)\big(\Phi'_\tau(Q(\tau))\big)^2}.
\end{align}
The rest of the analysis is similar to that given in Section \ref{bounds_static} by taking $\Phi_{\tau}(x) = e^{\lambda_\tau x}-1$. We now choose $\lambda_\tau = \frac{1}{4\sqrt{\tau(1+\mathcal{P}_\tau)}\,\sqrt{\log(\tau)+1}\,(\log(\log(\tau)+1)+1)}.$
We formally state our results in the following theorem,
\begin{theorem} \label{main_result2}
For the COCO problem with adversarially chosen $G$-Lipschitz cost and constraint functions, Algorithm \ref{coco_alg_dynamic} yields the following worst-case dynamic regret and CCV bounds for any $t \geq 1$ 
%for any $T \geq 1:$
\begin{align*}
 &\textsc{D-Regret}_t (\mathcal{P}_t) \leq 2GD\sqrt{1+\mathcal{P}_t}(\sqrt{t}+1),\\ &\textrm{CCV}_t\leq 8GD\sqrt{1+\mathcal{P}_t}\sqrt{t}\,\sqrt{\log t+1}\,\bigl(\log(\log t+1)+1\bigr)\log 4t,
 \end{align*}
 where $\mathcal{P}_t$ is the path-length of the worst-case comparator sequence $\{x_t^\star\}_{\tau =1}^t$ up to time $t.$
\end{theorem}
\section{An Adaptive Optimistic Algorithm}
Generalizing the seminal work of \citet{sinha2024optimal}, \citet{lekeufack2025optimisticalgorithmonlineconvex} presents an optimistic meta-algorithm for Constrained Online Convex Optimization (COCO) that leverages predictions of the loss and constraint functions to achieve superior performance when those predictions are accurate. However, their algorithm requires tuning a crucial parameter that depends on the total prediction error over the entire horizon -  a quantity that is unknown in advance. Consequently, to make their algorithm adaptive, they resort to using the doubling trick. In this section, we propose a modified algorithm that is both optimistic and continuously adaptive, obviating the use of the doubling trick.
\subsection{Optimistic meta-algorithm}
%Consistent with the structure of the preceding sections, we first introduce a base Online Convex Optimization (OCO) algorithm, followed by the Constrained Online Convex Optimization (COCO) meta-algorithm designed to utilize this base algorithm as a subroutine. 
\begin{algorithm}[H] % [H] suggests to place the algorithm "here"
\caption{Optimistic COCO}
\label{alg:coco_meta}
\begin{algorithmic}[1]
    \STATE {\bfseries Input:} Sequence of convex cost functions $\{f_t\}_{t=1}^T$ and constraint functions $\{g_t\}_{t=1}^T,$ an upper bound $G$ to the norm of their (sub)-gradients, Diameter $D$ of the admissible set $\mathcal{X}$
  \STATE {\bfseries Initialization:} Set $ x_1=0, Q(0)=0, Q(1)=0.$
    \FOR{round $t=1 \dots T$}
        \STATE Choose $x_t,$ observe $f_t, g_t,$ incur a cost of $f_t(x_t)$ and constraint violation of $(g_t(x_t))^+$
        \STATE $\tilde{f}_t \gets \alpha f_t, \tilde{g}_t \gets \alpha \max(0,g_t).$
        \STATE Predict $\bar{f}_{t+1}$ and $\bar{g}_{t+1}$.
        \STATE Compute $\hat{f}_t$ as per \eqref{eq:optim_surr}.
        \STATE Update $Q(t+1) = \frac{\lambda_t}{\lambda_{t+1}}Q(t) + \tilde{g}_t(x_t)$.
        \STATE Compute prediction $\bar{\hat{f}}_{t+1}$ as in \eqref{eq:pred_surr}.
        \STATE Pass $\hat{f}_t$ and $\bar{\hat{f}}_{t+1}$ to Algorithm \ref{alg:oomd}.
    \ENDFOR
\end{algorithmic}
\end{algorithm}
The above algorithm is a synthesis of the Optimistic COCO meta-algorithm proposed in \citet{lekeufack2025optimisticalgorithmonlineconvex} with our novel methodology of time-varying Lyapunov functions and modified queue recursion. It operates by constructing the true and predicted surrogate cost functions ($\hat{f}_t$ and $\bar{\hat{f}}_{t+1}$) using the adaptive queue $Q(t)$ updated via our multiplicative rule, and then passes these functions to a base optimistic OCO algorithm, specifically Optimistic Online Mirror Descent (Algorithm 5), which is detailed in Appendix D.1. The regret guarantee for this base algorithm is given in Theorem D.1. The Optimistic Online Mirror Descent algorithm was first proposed by \citet{rakhlin2014onlinelearningpredictablesequences} and we just restate the algorithm and its associated regret guarantee in the Appendix. In the following, we provide a brief overview of the optimistic setting, defining the prediction errors and the predicted surrogate cost function used by our algorithm, before stating the main theorem which guarantees its regret and constraint violation bounds. The complete proof is deferred to Appendix D.2.

In the optimistic setting, we assume that at the end of step $t$, the learner can make predictions $\bar{f}_{t+1}$ and $\bar{g}_{t+1}$. 
More precisely, we are interested in predictions of the gradients, and, for any function $h$, we denote 
by $\nabla\bar{h}_t$ the prediction of the gradient of $h$. We denote by $\bar{h}_t$ the function whose 
gradient is $\nabla\bar{h}_t$. Moreover, we define the following prediction errors
\begin{align}
\label{eq:pred_error}
    \epsilon_t(h) &:= ||\nabla h_t(x_t) - \nabla \bar{h}_t(x_t)||_{*}^2, \notag \\
    \mathcal{E}_t(h) &:= \sum_{\tau=1}^{t} \epsilon_{\tau}(h),
\end{align}

In the previous algorithms, we formed $\hat{f}_t$ at the end of round $t$ using the revealed cost $f_t$ and constraint $g_t$. Here, we will also form the predicted surrogate cost $\bar{\hat{f}}_{t+1}$ using the predicted cost $\bar{f}_{t+1}$ and constraint $\bar{g}_{t+1}$. We assume without loss of generality that both the actual and predicted functions have the same Lipchitz constant, which upon pre-processing is $(2D)^{-1}$. This also implies that both $\epsilon_t(\tilde{f})$ and $\epsilon_t(\tilde{g})$ are bounded from above by $D^{-2}$.

We define the predicted surrogate cost $\bar{\hat{f}}_{t+1}$ as-
\begin{align}
    \label{eq:pred_surr}\bar{\hat{f}}_{t+1} = \bar{f}_{t+1} + \Phi_{t+1}'(Q(t+1))\bar{g}_{t+1}
\end{align}
 However, note that if we were to retain the usual definition of $Q(t)$, the above equation would imply that we know the value of $\tilde{g}_{t+1}(x_{t+1})$ at the end of round $t$. This is not possible, therefore, \citet{lekeufack2025optimisticalgorithmonlineconvex} has to use delayed updates, wherein,  
\begin{align}
\label{optim-q-recursion}
    Q(t+1) = Q(t) + \tilde{g}_t(x_t)
\end{align}
To apply our adaptive technique we have to augment this with our multiplicative factor to get,
\begin{align}
     Q(t+1) = \frac{\lambda_{t}}{\lambda_{t+1}}Q(t) + \tilde{g}_t(x_t)
\end{align}
Note that, this means $\textrm{CCV}_t\leq Q(t+1)$. Further, $\lambda_{t+1}$ should be such that it is known after round $t$.

Also note that for parity with the definition of predicted surrogate loss, we need to modify our definition of surrogate loss. 
\begin{align}
\label{eq:optim_surr}\hat{f}_{t} = \tilde{f}_{t} + \Phi_{t}'(Q(t))\tilde{g}_{t}
\end{align}
where the $Q(t)$ now has been defined according to Eqn. \eqref{optim-q-recursion}.

We further choose our $\lambda_\tau$ whose value will be justified in subsequent analysis.
\iffalse
Using similar analysis as before,

\begin{align*}
\Phi_{\tau+1}(Q(\tau+1)) - \Phi_{\tau}(Q(\tau)) \leq \Phi'_{\tau+1}(Q(\tau+1)) \tilde{g}_\tau(x_\tau)
\end{align*}
Further note that, using 
\begin{align*}
    \Phi'_{\tau+1}(Q(\tau+1)) \tilde{g}_\tau(x_\tau)
     \leq  e\Phi'_{\tau}(Q(\tau)) \tilde{g}_\tau(x_\tau) 
\end{align*}
Therefore,
\begin{align}
    \Phi_{\tau+1}(Q(\tau+1)) - \Phi_{\tau}(Q(\tau)) \leq e\Phi'_{\tau}(Q(\tau)) \tilde{g}_\tau(x_\tau)
\end{align}
Overall, we get,
\begin{align}
    \Phi_{t+1}(Q(t+1)) +  \textrm{Regret}(x^\star) \leq \textrm{Regret}'(x^\star) 
\end{align}

We can bound $\textrm{Regret}'(x^\star)$ using Theorem \ref{opt-reg-bd}.

\begin{align}
\label{opt-rhs}
    \mathcal{E}_t(\hat{f}) &\leq \sqrt{\sum_{\tau=1}^t 2 \epsilon_\tau(f) + \sum_{\tau=1}^t 2\Phi'(Q(\tau))^2 \epsilon_\tau(\tilde{g})} \nonumber \\ &\leq \sqrt{2\mathcal{E}_t(f)} + \Phi_t(Q(t)) \sqrt{\sum_{\tau=1}^t \lambda_\tau^2 \epsilon_\tau(\tilde{g})}
\end{align}
\fi
\begin{align}
\label{eq:opt_lambda}
    &\lambda_{\tau} = \frac{1}{20(\sqrt{\frac{B}{\beta}}+\frac{B}{\beta})\sqrt{\gamma_\tau + 1} \sqrt{\log(\gamma_\tau + 1)+1} (\log(\log(\gamma_\tau+1)+1)+1)}
\end{align}.
where
\begin{align}
\label{eq:opt_gamma}
    \gamma_\tau = \mathcal{E}_{\tau-1}(\tilde{g}) + D^{-2}  
\end{align}
%Note that our choice of $\lambda_{\tau}$ implies that $\sum_{\tau=1}^t \lambda_\tau^2 \epsilon_\tau(g^+) \leq \frac{1}{4}$ and the subsequent analysis is similar to that in Section \ref{sec:analysis}. 
We summarize our results in the following theorem.
\begin{theorem}
\label{thm:opt-reg-vio}
For the COCO problem with adversarially chosen G-Lipschitz cost and constraint functions, Algorithm \ref{alg:coco_meta}, using Algorithm \ref{alg:oomd} as its base OCO algorithm, yields the following Regret and CCV bounds for any $t \in [T]$ where $T \ge 1$ is the horizon length:
\begin{align*}
    &\textrm{Regret}_t \le O\left(\sqrt{ \mathcal{E}_t(\tilde{f})}\right), \\
    &\textrm{CCV}_t \le \tilde{O}\left(\sqrt{ \mathcal{E}_{t}(\tilde{g})}\right).
\end{align*}
\end{theorem}
For completeness we provide the proof of the above theorem in section \ref{pf:opt-reg-vio}.
\iffalse
By convexity of $\Phi$, for any $\tau \geq 1$:
    \begin{align*}
        \Phi(Q_{\tau+1}) &\leq \Phi(Q_\tau) + \Phi'(Q_{\tau+1})\cdot (Q_{\tau+1} - Q_\tau) \\
            &= \Phi(Q_\tau) + \Phi'(Q_{\tau+1})\cdot g^+_t(x_\tau).
    \end{align*}
    Let $u\in \mathcal{X}^\star$, then by definition $g^+_\tau(u) = 0, \forall \tau \geq 1$, thus
    \begin{align*}
        & \Phi(Q_{\tau+1}) - \Phi(Q_\tau) + (f_\tau(x_\tau) - f_\tau(u)) \\
        &\leq \Phi'(Q_{\tau+1})g^+_\tau(x_\tau) +  (f_\tau(x_\tau) - f_\tau(u)) \\
        &\leq  f_\tau(x_\tau) + \Phi'(Q_\tau) g^+_\tau(x_\tau) \\
        &\quad \quad - \big(( f_\tau(u) + \Phi'(Q_\tau) g^+_\tau(u)\big) \\
        & \quad \quad + g^+_\tau(x_\tau)(\Phi'(Q_{\tau+1}) - \Phi'(Q_\tau)) \\
        &\leq \calL_\tau(x_\tau) - \calL_\tau(u) + g^+_\tau(x_\tau)(\Phi'(Q_{\tau+1}) - \Phi'(Q_\tau)).
    \end{align*}
    where 
    \[
        S_t = \sum_{\tau=1}^t g^+_\tau(x_\tau)(\Phi'(Q_{\tau+1}) - \Phi'(Q_\tau)).
    \]
\fi
\section{Simulations}
We demonstrate the practical efficacy of our algorithm via numerical experiments on the Online Shortest Path problem with delay constraints - a constrained version of the classic online shortest path problem. We perform numerical simulations on a semi-synthetic dataset constructed from real-world network measurements to model dynamic network conditions with adversarial costs and constraints. Our proposed anytime algorithm is compared against two key baselines: the fixed-horizon algorithm from \citet{sinha2024optimal}, which requires advance knowledge of the total horizon $T$, and its adaptation to an unknown horizon setting via the standard doubling trick. The results, presented in Figures 1 and 2 in the appendix, empirically validate our theoretical claims. The plots show that our anytime algorithm consistently outperforms the doubling trick baseline in both cumulative regret and constraint violation, while also avoiding the instability caused by its periodic restarts. Notably, our algorithm also achieves superior practical performance compared to the fixed-horizon baseline, demonstrating the advantages of its adaptive parameterization. We have detailed the full experimental setup, dataset construction, and analysis of the results in Appendix \ref{expts}.
\section{Conclusion}
\label{sec:conclusion}
In this work, we proposed anytime algorithms for online convex optimization with adversarial constraints (COCO) applicable across various problems, including the static regret, dynamic regret, and optimistic settings. Our approach employs time-varying Lyapunov functions coupled with a novel multiplicative factor in the virtual queue recursion, a key technical innovation ensuring essential monotonicity properties. This methodology circumvents the need for impractical techniques like the doubling trick and its associated restarts. The resulting algorithms are more practical and stable, benefits which we confirm through numerical simulations on the constrained online shortest path problem.

\newpage
%\bibliography{ou}
%\bibliographystyle{plainnat}

\newpage
\appendix
\section{The Doubling Trick}
\label{sec:doubling_trick}
 The doubling trick is a standard technique for adapting online algorithms that require a known time horizon $T$ to a setting where the horizon is unknown. The method works by running the algorithm in exponentially increasing phases. At the start of each new phase, the algorithm is restarted, and its internal parameters (which depend on $T$) are recalculated using the new, doubled phase length as the estimate for $T$.

However, as noted in several prior works, this technique, while theoretically functional, is often criticized for being ``aesthetically inelegant and impractical'' \citep{luo2014towards, zhang2024improving, kwon2014continuous}. A major line of critique, articulated by \citet{luo2014towards}, is that the method is ``intuitively wasteful, since it repeatedly restarts itself, entirely forgetting all the preceding information''. Their work, which proposes an alternative adaptive algorithm based on a ``pretend prior distribution'' over the horizon, also demonstrated empirically that the doubling trick is ``beaten by most of the other algorithms'' in practice.

Similarly, \citet{kwon2014continuous} developed their continuous-time approach specifically to provide a ``unified any-time analysis without needing to reboot the algorithm every so often''. Their work shows that by using a \emph{time-varying parameter} (e.g., $\eta_t \propto 1/\sqrt{t}$) rather than a fixed one, one can achieve optimal $\mathcal{O}(t^{-1/2})$ regret bounds without resorting to restarts.

More recently, \citet{zhang2024improving} also motivated their work by designing an algorithm that ``does not employ the impractical doubling trick''. They reiterated that restarting ``wastes data'' and causes ``large jumps" in the decision sequence, which can be undesirable, and that it performs ``considerably worse'' in practice. As a solution, they proposed a ``refined discretization argument'' from a continuous-time model to preserve adaptivity without restarts.

The consensus from these works is that continuously adaptive algorithms are generally preferable and perform better in practice. Our own experiments in Appendix~\ref{expts} (Figures~\ref{fig:ccv_comparison} and \ref{fig:regret_comparison}) confirm this finding, showing that our intrinsically adaptive method offers distinctive practical performance gains over the standard doubling trick baseline.
\section{Supporting Lemmas}

\begin{lemma}
\label{lemma:step-size-special}
For 
\[
\lambda_\tau = \frac{1}{4\sqrt{\tau}\,\sqrt{\log(\tau)+1}\,\bigl(\log(\log(\tau)+1)+1\bigr)},
\]
we have
\[
\sum_{\tau=1}^t \lambda_\tau^2 < \tfrac{1}{4}, \quad \forall\, t \geq 1.
\]
\end{lemma}

\begin{proof}
We observe
\[
\lambda_\tau^2
= \frac{1}{16 \, \tau \, (\log(\tau)+1)\,\bigl(\log(\log(\tau)+1)+1\bigr)^2}.
\]

Define
\[
a(x) = \frac{1}{16 \, x \, (\log(x)+1)\,\bigl(\log(\log(x)+1)+1\bigr)^2}, \quad x \geq 1.
\]
The function $a(x)$ is positive and eventually decreasing, so we can apply the integral test:
\[
\sum_{\tau=1}^t \lambda_\tau^2 \;\leq\; a(1) + \int_{1}^{t} a(x)\, dx.
\]

\medskip
\noindent
We compute
\[
\int_{1}^{t} a(x)\, dx
= \frac{1}{16} \int_{1}^{t} \frac{dx}{x (\log x+1)(\log(\log x+1)+1)^2}.
\]
With the substitution $u = \log x + 1$ (so $du = dx/x$), this becomes
\[
= \frac{1}{16} \int_{1}^{\log t + 1} \frac{du}{u(\log u+1)^2}.
\]
Now let $v = \log u + 1$ so that $dv = du/u$. Then
\[
= \frac{1}{16} \int_{1}^{\log(\log t+1)+1} \frac{dv}{v^2}
= \frac{1}{16}\left(1 - \frac{1}{\log(\log t+1)+1}\right).
\]

\medskip
\noindent
Hence
\[
\sum_{\tau=1}^t \lambda_\tau^2
\leq a(1) + \frac{1}{16}\left(1 - \frac{1}{\log(\log t+1)+1}\right).
\]
Explicitly,
\[
a(1) = \frac{1}{16 \cdot 1 \cdot (\log 1 + 1)(\log(\log 1 + 1)+1)^2} = \frac{1}{16}.
\]
Therefore,
\[
\sum_{\tau=1}^t \lambda_\tau^2
\leq \frac{1}{16} + \frac{1}{16}\cdot 1
= \frac{1}{8}
 < \frac{1}{4}.
\]

\medskip
\noindent
The bound holds uniformly for all $t \geq 1$, proving the claim.
\end{proof}
\section{Proof of Theorem \ref{dyn_reg_ogd}}
\label{dyn_reg_ogd_proof}
Define $y_{t+1} \vcentcolon=  x_t - \eta_t \nabla \hat{f}_t (x_t),$ where the non-increasing adaptive step size sequence $\{\eta_t\}_{t \geq 1}$ has been defined in Algorithm \ref{ogd_alg}. For any feasible comparator action $x^\star_t \in \mathcal{X}$, we have
\begin{eqnarray*}
||x_{t+1} - x^\star_t||^2 
\stackrel{(a)}{\leq} ||y_{t+1} - x^\star_t||^2 
= ||x_t - x^\star_t||^2 + \eta_{t}^2 ||\nabla_t||^2 - 2 \eta_t \nabla_{t}^\top (x_t - x^\star_t),
\end{eqnarray*}
where inequality (a) follows from the non-expansive property of Euclidean projection and the second equality follows from the definition of $y_{t+1}$.
Rearranging the above inequality, we have:
\begin{equation}
    2 \nabla_{t}^\top(x_t - x^\star_t) \leq \frac{||x_t - x^\star_t||^2 - ||x_{t+1} - x^\star_t||^2 }{\eta_t} + \eta_t ||\nabla_t||^2.
\end{equation}

Using the convexity of the cost functions and summing the above inequalities over $1\leq t \leq T$, we obtain:
\begin{eqnarray} \label{dyn-reg-eq}
2 \sum_{t=1}^{T} \left( \hat{f}_t(x_t) - \hat{f}_t(x^\star_t) \right) 
\leq 2\sum_{t=1}^T \nabla_{t}^\top(x_t - x^\star_t) 
\leq \underbrace{\sum_{t=1}^T \frac{||x_t - x^\star_t||^2 - ||x_{t+1} - x^\star_t||^2 }{\eta_t}}_{(A)} + \sum_{t=1}^T \eta_t ||\nabla_t||^2. 
\end{eqnarray}
%\cmt{Since $\eta_t$ is monotone, the result should be immediate. See Hazan.}
%The RHS of the expression in \eqref{dyn-reg-eq} can be rewritten as
Next we simplify term (A) in Eqn. \eqref{dyn-reg-eq}. It can be expressed as
\begin{flalign} \label{termB}
    & \frac{||x^\star_1 - x_1||^2}{\eta_1} - \frac{||x^\star_T - x_{T+1}||^2}{\eta_T} + \sum\limits_{t=1}^{T-1} \frac{||x_{t+1} - x^\star_{t+1}||^2}{\eta_{t+1}} - \frac{||x_{t+1} - x^\star_t||^2}{\eta_t}  \nonumber\\
    & =\frac{||x^\star_1 - x_1||^2}{\eta_1} - \frac{||x^\star_T - x_{T+1}||^2}{\eta_T} +
     \underbrace{\sum\limits_{t=1}^{T-1} \frac{\eta_t||x_{t+1} - x^\star_{t+1}||^2 - \eta_{t+1}||x_{t+1} - x^\star_t||^2}{\eta_t\eta_{t+1}}}_\text{\clap{(B)}}.
    % &\leq \sum\limits_{t=1}^T \frac{||x_t + x_{t+1} - 2x^\star_t||\quad||x_t - x_{t+1}||}{\eta_t} + \sum\limits_{t=1}^T \eta_t ||\nabla_t||^2
    % && \parbox[t]{0.5\textwidth}{\raggedleft (Using Cauchy-Schwartz)} \\
    % &\leq 2D\sum\limits_{t=1}^T \frac{||x_t - x_{t+1}||}{\eta_t} + \sum\limits_{t=1}^T \eta_t ||\nabla_t||^2
    % && \parbox[t]{0.5\textwidth}{\raggedleft (Using triangle inequality and upper bounding by the diameter of the feasible set)} \\
\end{flalign}
We next upper bound term (B) in \eqref{termB}. 
\begin{flalign*}
    &\sum\limits_{t=1}^{T-1} \frac{||\sqrt{\eta_t}x_{t+1} - \sqrt{\eta_t}x^\star_{t+1}||^2 - ||\sqrt{\eta_{t+1}}x_{t+1} - \sqrt{\eta_{t+1}}x^\star_t||^2}{\eta_t\eta_{t+1}} \\
    &= \sum\limits_{t=1}^{T-1} \frac{\langle (\sqrt{\eta_t} + \sqrt{\eta_{t+1}})x_{t+1} - \sqrt{\eta_t}x^\star_{t+1} - \sqrt{\eta_{t+1}}x^\star_t, (\sqrt{\eta_t} - \sqrt{\eta_{t+1}})x_{t+1} - \sqrt{\eta_t}x^\star_{t+1} + \sqrt{\eta_{t+1}}x^\star_t \rangle}{\eta_t\eta_{t+1}} \\
    &\leq \sum\limits_{t=1}^{T-1} \frac{||(\sqrt{\eta_t} + \sqrt{\eta_{t+1}})x_{t+1} - \sqrt{\eta_t}x^\star_{t+1} - \sqrt{\eta_{t+1}}x^\star_t|| \: ||(\sqrt{\eta_t} - \sqrt{\eta_{t+1}})x_{t+1} - \sqrt{\eta_t}x^\star_{t+1} + \sqrt{\eta_{t+1}}x^\star_t||}{\eta_t\eta_{t+1}},   
\end{flalign*}
where the last step follows from an application of the Cauchy-Schwarz inequality. Note that the first term in the numerator can be bounded as:
%\begin{flalign*}
  \[  ||(\sqrt{\eta_t} + \sqrt{\eta_{t+1}})x_{t+1} - \sqrt{\eta_t}x^\star_{t+1} - \sqrt{\eta_{t+1}}x^\star_t|| 
    \leq \sqrt{\eta_t}||x_{t+1} - x^\star_{t+1}|| + \sqrt{\eta_{t+1}} ||x_{t+1} - x^\star_t||
    \leq (\sqrt{\eta_t} + \sqrt{\eta_{t+1}}) D ,\]
%\end{flalign*}
where we have used the triangle inequality and the feasibility of the algorithm's and benchmark's actions in the last step. 
Using this, we have
\begin{flalign*}
 (B)   &\leq D\sum\limits_{t=1}^{T-1} \frac{ (\sqrt{\eta_t} + \sqrt{\eta_{t+1}}) \: ||(\sqrt{\eta_t} - \sqrt{\eta_{t+1}})x_{t+1} - \sqrt{\eta_t}x^\star_{t+1} + \sqrt{\eta_{t+1}}x^\star_t||}{\eta_t\eta_{t+1}}\\
    &= D\sum\limits_{t=1}^{T-1} \frac{ (\sqrt{\eta_t} + \sqrt{\eta_{t+1}}) ||(\sqrt{\eta_t} - \sqrt{\eta_{t+1}})(x_{t+1}-x^\star_{t+1}) + \sqrt{\eta_{t+1}}(x^\star_t - x^\star_{t+1})||}{\eta_t\eta_{t+1}}\\
    %[(\sqrt{\eta_t} - \sqrt{\eta_{t+1}})||x_{t+1}|| + ||\sqrt{\eta_{t+1}}x^\star_t - \sqrt{\eta_t}x^\star_{t+1}||]}{\eta_t\eta_{t+1}} \\
%    && \parbox[t]{0.2\textwidth}{\raggedleft (Using triangle inequality)} \\
    &\stackrel{(a)}{\leq} D\sum\limits_{t=1}^{T-1} \frac{ (\sqrt{\eta_t} + \sqrt{\eta_{t+1}})  \: 
    [(\sqrt{\eta_t} - \sqrt{\eta_{t+1}})D + \sqrt{\eta_{t+1}}||x^\star_t - x^\star_{t+1}||]}{\eta_t\eta_{t+1}}\\
%\end{flalign*}
 %Now, 
%
%\begin{flalign*}
%    &||\sqrt{\eta_{t+1}}x^\star_t - \sqrt{\eta_t}x^\star_{t+1}|| \\
%    &=||\sqrt{\eta_{t+1}}x^\star_t - \sqrt{\eta_{t+1}}x^\star_{t+1} + \sqrt{\eta_{t+1}}x^\star_{t+1} - \sqrt{\eta_t}x^\star_{t+1}|| \\
%    &\leq \sqrt{\eta_{t+1}}||x^\star_{t+1} - x^\star_{t}|| + (\sqrt{\eta_{t+1}} - \sqrt{\eta_{t}})||x^\star_{t+1}|| \\ 
%    &\leq \sqrt{\eta_{t+1}}||x^\star_{t+1} - x^\star_{t}|| + (\sqrt{\eta_{t}} - \sqrt{\eta_{t+1}})\frac{D}{2}
%\end{flalign*}
%Using this, term (a)
%\begin{flalign*}
   % &\leq \sum\limits_{t=1}^{T-1} \frac{ (\sqrt{\eta_t} + \sqrt{\eta_{t+1}}) D \: 
    %[(\sqrt{\eta_t} - \sqrt{\eta_{t+1}})\frac{D}{2} + (\sqrt{\eta_{t}} - \sqrt{\eta_{t+1}})\frac{D}{2} + \sqrt{\eta_{t+1}}||x^\star_{t+1} - x^\star_{t}||]}{\eta_t\eta_{t+1}} \\
    %&=\sum\limits_{t=1}^{T-1} \frac{ (\sqrt{\eta_t} + \sqrt{\eta_{t+1}}) D \: 
    %[(\sqrt{\eta_t} - \sqrt{\eta_{t+1}})D + \sqrt{\eta_{t+1}}||x^\star_{t+1} - x^\star_{t}||]}{\eta_t\eta_{t+1}} \\
    &\stackrel{(b)}{\leq} D\sum\limits_{t=1}^{T-1} \frac{(\eta_t - \eta_{t+1}) D \: 
    + 2\eta_t ||x^\star_{t+1} - x^\star_{t}||}{\eta_t\eta_{t+1}} \\
    %&\leq \sum\limits_{t=1}^{T-1} \frac{\eta_t D^2 \: 
    %+ 2\eta_t D||x^\star_{t+1} - x^\star_{t}||}{\eta_t\eta_{t+1}} \\
    &= D^2 \big(\frac{1}{\eta_T} - \frac{1}{\eta_1}\big) + 2D \sum\limits_{t=1}^{T-1} \frac{||x^\star_{t+1} - x^\star_{t}||}{\eta_{t+1}}\\
    &\stackrel{(c)}{\leq}  D^2 \big(\frac{1}{\eta_T} - \frac{1}{\eta_1}\big) + \frac{2D \mathcal{P}_T(x^\star_{1:T})}{\eta_T},
    % &= \sum\limits_{t=1}^{T-1} \frac{ (\sqrt{\eta_t} + \sqrt{\eta_{t+1}}) D \:
    % \sqrt{\eta_{t+1}}||x^\star_{t+1} - x^\star_{t}||}
    % {\eta_t\eta_{t+1}}\\
    % &\leq \sum\limits_{t=1}^{T-1} \frac{ 2\eta_t D \: ||x^\star_{t+1} - x^\star_{t}||}
    % {\eta_t\eta_{t+1}} 
    % && \parbox[t]{0.1\textwidth}{\raggedleft (Using \(\eta_{t+1} \leq \eta_{t})\)} \\
    % &\leq 2 D \sum\limits_{t=1}^{T-1} \frac{||x^\star_{t+1} - x^\star_{t}||}
    % {\eta_{t+1}} \\
    % &\leq \frac{2 D}{\eta_{T}} \sum\limits_{t=1}^{T-1} ||x^\star_{t+1} - x^\star_{t}|| \\
    % &= \frac{2 D \mathcal{P}^T}{\eta_{T}}
\end{flalign*}
where in step (a), we have used the triangle inequality combined with the feasibility of the algorithm's and benchmark's actions, and in step (c), we have used the definition of the path length of the comparator. The non-increasing property of the step sizes, \emph{i.e.,} $\eta_t \geq \eta_{t+1}, \forall t \geq 1$ was used in steps (a), (b), and (c). Finally, combining the above bound with Eqns. \eqref{dyn-reg-eq} and \eqref{termB}, we conclude 
\begin{flalign*}
    & 2(\sum\limits_{t=1}^{T} \hat{f}_t(x_t) - \hat{f}_t(x^\star_t)) \\
    &\leq \frac{||x^\star_1 - x_1||^2}{\eta_1} - \frac{||x^\star_T - x_{T+1}||^2}{\eta_T} + D^2 \big(\frac{1}{\eta_T} - \frac{1}{\eta_1}\big)
      + \frac{2 D \mathcal{P}_T(x^\star_{1:T})}{\eta_{T}}
      + \sum\limits_{t=1}^T \eta_t ||\nabla_t||^2\\
   % &\leq \frac{D^2}{\eta_T}
    %  + \frac{2 D \mathcal{P}^T}{\eta_{T}}
    %  + \sum\limits_{t=1}^T \eta_t ||\nabla_t||^2\\
    &\leq \frac{D^2 + 2 D \mathcal{P}_T(x^\star_{1:T})}{\eta_T}
      + \sum\limits_{t=1}^T \eta_t ||\nabla_t||^2.
\end{flalign*}
Hence, the dynamic regret of Algorithm \ref{ogd_alg} can be upper bounded as 
\begin{equation}\label{eqn2}
    \textsc{D-Regret}_T \leq \frac{\max(D^2,2D)(1 + \mathcal{P}_T(x^\star_{1:T}))}{2\eta_T}
      + \underbrace{\frac{1}{2}\sum\limits_{t=1}^T \eta_t ||\nabla_t||^2}_{(C)}.
\end{equation}
The dynamic regret bound in Eqn.\ \eqref{eqn2} holds for Algorithm \ref{ogd_alg} with any non-increasing step sizes. Assuming the path-length is known to be bounded as $\mathcal{P}_T(x^\star_{1:T}) \leq \mathcal{P}_T,$ using the specific choice of the step size sequence $\eta_t = \frac{(D+1) (1+\mathcal{P}_t)}{ \sqrt{2\sum_{\tau=1}^t (1+\mathcal{P}_\tau)||\nabla_\tau||^2}} t \geq 1,$  we can upper bound term (C) as follows 
  \begin{flalign*}
\frac{1}{2}\sum\limits_{t=1}^T \eta_t ||\nabla_t||^2 &=
\frac{(D+1)}{2\sqrt{2}}  \sum_{t=1}^T\frac{(1+\mathcal{P}_t)||\nabla_t||^2}{\sqrt{2\sum_{\tau=1}^t (1+\mathcal{P}_\tau)||\nabla_\tau||^2} } \\
&\leq \frac{(D+1)}{2\sqrt{2}} \int_{0}^{\sum\limits_{t=1}^T (1+\mathcal{P}_t)||\nabla_t||^2} \frac{dz}{\sqrt{z}} \\
&= \frac{(D+1)}{\sqrt{2}}\sqrt{\sum\limits_{t=1}^T (1+\mathcal{P}_t)||\nabla_t||^2}.
\end{flalign*}
Hence, from \eqref{eqn2}, the dynamic regret for Algorithm \ref{ogd_alg} can be bounded as 
\begin{equation}\label{eqn3}
    \textsc{D-Regret}_T(\mathcal{P}_T) \leq \sqrt{2}(D+1)
    \sqrt{\sum\limits_{t=1}^T (1+\mathcal{P}_t)||\nabla_t||^2}.
\end{equation}
\section{Deferred technical details related to the optimistic setting}
\subsection{Base Optimistic OCO algorithm and regret guarantee}
\begin{algorithm}[H] % [H] suggests to place the algorithm "here"
\caption{Optimistic Online Mirror Descent}
\label{alg:oomd}
\begin{algorithmic}[1]
    \STATE \textbf{Input} : Convex decision set $\mathcal{X}$, sequence of convex cost functions $\{\hat{f}_t\}_{t=1}^{T}$, $\textrm{diam}(\mathcal{X}) = D.$
    \STATE \textbf{Initialize} : $x_1 \in \mathcal{X}$ arbitrarily.
    \FOR{round $t=1 \dots T$}
        \STATE Play action $x_t$, receive $\hat{f}_t$. Compute $\nabla_t = \nabla \hat{f}_t(x_t)$.
        \STATE Compute  $\eta_{t+1} = \min\left\{\frac{\sqrt{\beta B}}{\sqrt{\mathcal{E}_{t}(\hat{f})} + \sqrt{\mathcal{E}_{t-1}(\hat{f}})}, \frac{\beta}{\max_{1\leq\tau\leq t+1}\bar{L}_\tau}\right\}$
        \STATE $\tilde{x}_{t+1} := \argmin_{x \in \mathcal{X}} \left\langle \nabla_t, x \right\rangle + \frac{1}{\eta_t} B^R(x; \tilde{x}_t)$.
        \STATE Compute $\bar{\hat{f}}_{t+1} = \nabla \bar{\hat{f}}_{t+1}(\tilde{x}_{t+1})$.
        \STATE $x_{t+1} := \argmin_{x \in \mathcal{X}} \left\langle \bar{\hat{f}}_{t+1}, x \right\rangle + \frac{1}{\eta_{t+1}} B^R(x; \tilde{x}_{t+1})$.
    \ENDFOR
\end{algorithmic}
\end{algorithm}
\begin{theorem}{\citep[Theorem 10]{lekeufack2025optimisticalgorithmonlineconvex}}
\label{thm:opt-reg-bd}
Let Algorithm \ref{alg:oomd} be run on a sequence of cost functions $\{\hat{f}_t\}_{t=1}^T$. Let the corresponding sequence of predicted costs be $\{\tilde{\hat{f}}_t\}_{t=1}^T$. If $\bar{L}_t$ is the Lipchitz constant of the $t$-th predicted loss, then upon running Algorithm \ref{alg:oomd} with the following learning rate schedule

$\eta_t = \min\left\{\frac{\sqrt{\beta B}}{\sqrt{\mathcal{E}_{t-1}(\hat{f})} + \sqrt{\mathcal{E}_{t-2}(\hat{f}})}, \frac{\beta}{\max_{1\leq \tau \leq t}\bar{L}_{\tau}}\right\}$ we get the following regret bound
$$
Regret_t \le 5\sqrt{\frac{B}{\beta}}\sqrt{\mathcal{E}_{t}(\hat{f})} + 5\frac{B}{\beta}\max_{1\leq \tau \leq t}\bar{L}_{\tau}
$$
where $\beta$ is the strong convexity parameter of the regularizer $R$ defining the Bregman divergence, B is an upper bound on this divergence over the algorithm's execution and $\mathcal{E}(\hat{f})$ has been defined in \eqref{eq:pred_error}.
\end{theorem}
\begin{remark}
    The original theorem statement from \citet{lekeufack2025optimisticalgorithmonlineconvex} assumed that $\bar{L}_t \leq \bar{L}_{t+1}$ and thus instead of $\max_{1\leq \tau \leq t}\bar{L}_{\tau}$ it only had $\bar{L}_t$. The assumption naturally held true in their case since they made use time-invariant Lyapunov functions as in \citet{sinha2024optimal} and this implied $\bar{L}_t = (1+\Phi'(Q(t)))G$. In our setting, since $\bar{L}_t = (1+\Phi_t'(Q(t)))G$, the Lipchitz constants may not be monotonically non-decreasing and to compensate for that, we replace the occurrence of $\bar{L}_t$ in the theorem statement with $\max_{1\leq \tau \leq t}\bar{L}_{\tau}$.
\end{remark}
\subsection{Proof of Theorem \ref{thm:opt-reg-vio}}
\label{pf:opt-reg-vio}
The proof follows a similar structure to the proof of Theorem 4.2 but incorporates the delayed queue update and the regret bound for the optimistic OCO algorithm.

Using similar analysis as before,

\begin{align*}
\Phi_{\tau+1}(Q(\tau+1)) - \Phi_{\tau}(Q(\tau)) \leq \Phi'_{\tau+1}(Q(\tau+1)) \tilde{g}_\tau(x_\tau)
\end{align*}

Let $x^* \in \mathcal{X}^*$ be a feasible comparator, so $g_\tau(x^*) \le 0$ for all $\tau$.
Consider the quantity $\Phi_{\tau+1}(Q(\tau+1)) - \Phi_{\tau}(Q(\tau)) + \tilde{f}_\tau(x_\tau) - \tilde{f}_\tau(x^*)$:
\begin{align*}
&\Phi_{\tau+1}(Q(\tau+1)) - \Phi_{\tau}(Q(\tau)) + \tilde{f}_\tau(x_\tau) - \tilde{f}_\tau(x^*) \\
\le &\Phi'_{\tau+1}(Q(\tau+1)) \tilde{g}_\tau(x_\tau)+(x_{\tau}) + \tilde{f}_\tau(x_\tau) - \tilde{f}_\tau(x^*)  \\
= &(\tilde{f}_\tau(x_\tau) +  \Phi'_{\tau}(Q(\tau)) \tilde{g}_{\tau}(x_{\tau}))  - (\tilde{f}_\tau(x^*) +  \Phi'_{\tau}(Q(\tau)) \tilde{g}_{\tau}(x^*)) + \Phi'_{\tau+1}(Q(\tau+1)) \tilde{g}_\tau(x_\tau) -  \Phi'_{\tau}(Q(\tau)) \tilde{g}_{\tau}(x_{\tau})) \\= &\hat{f}_\tau(x_\tau) - \hat{f}_\tau(x^*) + \Phi'_{\tau+1}(Q(\tau+1)) \tilde{g}_\tau(x_\tau) -  \Phi'_{\tau}(Q(\tau)) \tilde{g}_{\tau}(x_{\tau})) 
\end{align*}
Also note that, 
\begin{align*}
   \Phi'_{\tau+1}(Q(\tau+1)) \tilde{g}_\tau(x_\tau) -  \Phi'_{\tau}(Q(\tau)) \tilde{g}_{\tau}(x_{\tau})) &\leq \Phi'_{\tau+1}(Q(\tau+1))  -  \Phi'_{\tau}(Q(\tau)) \\ &= \lambda_{t+1}(\Phi_{\tau+1}(Q(\tau+1)) + 1) -  \lambda_{t}(\Phi_{\tau}(Q(\tau))+1) \\ &\leq \lambda_1 (\Phi_{\tau+1}(Q(\tau+1))-\Phi_{\tau}(Q(\tau))) \\ 
   &\leq \frac{1}{4}(\Phi_{\tau+1}(Q(\tau+1))-\Phi_{\tau}(Q(\tau)))
\end{align*}
where the first inequality follows because $\tilde{g}_{\tau}(x_\tau)\leq 1$, the second equality follows due to the form of $\Phi_t$, the third inequality follows due to the decreasing nature of $\lambda_t$ and the last inequality follows because $\lambda_1 \leq \frac{1}{4}$.
Therefore, we have,
\begin{align*}
    \Phi_{\tau+1}(Q(\tau+1)) - \Phi_{\tau}(Q(\tau)) + \tilde{f}_\tau(x_\tau) - \tilde{f}_\tau(x^*) \leq \hat{f}_\tau(x_\tau) - \hat{f}_\tau(x^*) + \frac{1}{4} (\Phi_{\tau+1}(Q(\tau+1))-\Phi_{\tau}(Q(\tau)))
\end{align*}
Summing from $\tau=1$ to $t$:
\[
\Phi_{t+1}(Q(t+1)) - \Phi_1(Q(1)) +  \sum_{\tau=1}^t (\tilde{f}_\tau(x_\tau) - \tilde{f}_\tau(x^*)) \le  \sum_{\tau=1}^t \hat{f}_\tau(x_\tau) - \hat{f}_\tau(x^*) + \frac{1}{4} \Phi_{t+1}(Q(t+1)) 
\]
Since $Q(1)=0$ and $\Phi_1(0)=0$, and recognizing the regret terms $\textrm{Regret}_t(x^*) = \sum_{\tau=1}^t (\tilde{f}_\tau(x_\tau) - \tilde{f}_\tau(x^*))$ and $\textrm{Regret}'_t(x^*) = \sum_{\tau=1}^t (\hat{f}_\tau(x_\tau) - \hat{f}_\tau(x^*))$,
\begin{align}
    \frac{3}{4}\Phi_{t+1}(Q(t+1)) +  \textrm{Regret}(x^\star) \leq \textrm{Regret}'(x^\star)
\end{align}

Where $\textrm{Regret}'_t(x^*)$ is the regret of the base optimistic OCO algorithm (Algorithm \ref{alg:oomd}) run on the sequence of surrogate costs $\{\hat{f}_\tau\}_{\tau=1}^t$. We bound $\textrm{Regret}'_t(x^*)$ using Theorem \ref{thm:opt-reg-bd}:
\[
\textrm{Regret}'_t(x^*) \le 5\sqrt{\frac{B}{\beta}}\sqrt{\mathcal{E}_{t}(\hat{f})} + 5\frac{B}{\beta}\max_{1\leq \tau \leq t}\bar{L}_{\tau}
\]
Substituting this back:
\begin{align}
\label{eq:opt-reg-decomp-1}
\frac{3}{4}\Phi_{t+1}(Q(t+1))+\textrm{Regret}_t(x^*) \le 5\sqrt{\frac{B}{\beta}}\sqrt{\mathcal{E}_{t}(\hat{f})} + 5\frac{B}{\beta}\max_{1\leq \tau \leq t}\bar{L}_{\tau} 
\end{align}

First, we bound the surrogate prediction error $\mathcal{E}_t(\hat{f})$ using Eq. (29):
\begin{align}
\label{opt-rhs}
    \sqrt{\mathcal{E}_t(\hat{f})} &= \sqrt{\sum_{\tau=1}^t (\hat{f}_\tau(x_\tau)-\bar{\hat{f}}_\tau(x_\tau))^2} = \sqrt{\sum_{\tau=1}^t (\tilde{f}_{t} + \Phi_{t}'(Q(t))\tilde{g}_{t}-\bar{f}_{t} - \Phi_{t}'(Q(t))\bar{g}_{t})^2}\nonumber \\
    &\leq \sqrt{\sum_{\tau=1}^t 2 (\tilde{f}_\tau(x_\tau)-\bar{f}_\tau(x_\tau))^2 + \sum_{\tau=1}^t 2\Phi'(Q(\tau))^2 (\tilde{g}_\tau(x_\tau)-\bar{g}_\tau(x_\tau))^2} \leq \sqrt{\sum_{\tau=1}^t 2 \epsilon_\tau(f) + \sum_{\tau=1}^t 2\Phi'(Q(\tau))^2 \epsilon_\tau(\tilde{g})}\nonumber \\ &\leq \sqrt{2\mathcal{E}_t(f)} + (\Phi_t(Q(t)) +1)\sqrt{\sum_{\tau=1}^t 2\lambda_\tau^2 \epsilon_\tau(\tilde{g})}
\end{align}
\begin{align*}
    5\sqrt{\frac{B}{\beta}}\sqrt{\mathcal{E}_{t}(\hat{f})} \le 5\sqrt{\frac{B}{\beta}} \sqrt{2\mathcal{E}_t(f)} + (\Phi_t(Q(t)) +1)5\sqrt{\frac{B}{\beta}}\sqrt{\sum_{\tau=1}^t 2\lambda_\tau^2 \epsilon_\tau(\tilde{g})}
\end{align*}
We choose $\lambda_\tau$ adaptively based on the prediction error of the constraints as defined before Theorem \ref{thm:opt-reg-vio}. Recall that,
\begin{align*}
    \lambda_{\tau} = \frac{1}{20(\sqrt{\frac{B}{\beta}}+\frac{B}{\beta})\sqrt{\gamma_\tau + 1} \sqrt{\log(\gamma_\tau + 1)+1} (\log(\log(\gamma_\tau+1)+1)+1)}
\end{align*}.
Note that, $\epsilon_\tau(\tilde{g}) \le D^{-2}$ (as pre-processed functions are $(2D)^{-1}$-Lipschitz), we can show that $\mathcal{E}_{\tau}(\tilde{g}) = \mathcal{E}_{\tau-1}(\tilde{g}) + \epsilon_\tau(\tilde{g}) \leq \mathcal{E}_{\tau-1}(\tilde{g}) + D^{-2} \leq \gamma_\tau.$
Thus:
\begin{align*}
    &\lambda_{\tau} \leq \frac{1}{20(\sqrt{\frac{B}{\beta}}+\frac{B}{\beta})\sqrt{\mathcal{E}_{\tau}(\tilde{g}) + 1} \sqrt{\log(\mathcal{E}_{\tau}(\tilde{g}) + 1)+1} (\log(\log(\mathcal{E}_{\tau}(\tilde{g})+1)+1)+1)} \\ \implies
    &\lambda_{\tau}^2 \leq \frac{1}{400(\sqrt{\frac{B}{\beta}}+\frac{B}{\beta})^2(\mathcal{E}_{\tau}(\tilde{g}) + 1) (\log(\mathcal{E}_{\tau}(\tilde{g}) + 1)+1) (\log(\log(\mathcal{E}_{\tau}(\tilde{g})+1)+1)+1)^2} \\ \implies
    &2\epsilon_\tau(\tilde{g})\lambda_{\tau}^2 \leq \frac{\epsilon_\tau(\tilde{g})}{200(\sqrt{\frac{B}{\beta}}+\frac{B}{\beta})^2(\mathcal{E}_{\tau}(\tilde{g}) + 1) (\log(\mathcal{E}_{\tau}(\tilde{g}) + 1)+1) (\log(\log(\mathcal{E}_{\tau}(\tilde{g})+1)+1)+1)^2} \\ \implies
    &2\epsilon_\tau(\tilde{g})\lambda_{\tau}^2 \leq \frac{(\mathcal{E}_{\tau}(\tilde{g}) + 1) - (\mathcal{E}_{\tau-1}(\tilde{g}) + 1)}{200(\sqrt{\frac{B}{\beta}}+\frac{B}{\beta})^2(\mathcal{E}_{\tau}(\tilde{g}) + 1) (\log(\mathcal{E}_{\tau}(\tilde{g}) + 1)+1) (\log(\log(\mathcal{E}_{\tau}(\tilde{g})+1)+1)+1)^2}
\end{align*}.
The above form makes it clear that $\sum_{\tau=1}^t 2\lambda_\tau^2 \epsilon_\tau(\tilde{g})$ is amenable to the integral-sum trick as in Lemma \ref{lemma:step-size-special}. Overall, we get, 
\begin{align*}
    \sum_{\tau=1}^t 2\lambda_\tau^2 \epsilon_\tau(\tilde{g}) \leq \frac{1}{100(\sqrt{\frac{B}{\beta}}+\frac{B}{\beta})^2} \implies 
    5\sqrt{\frac{B}{\beta}}\sqrt{\sum_{\tau=1}^t 2\lambda_\tau^2 \epsilon_\tau(\tilde{g})} \leq \frac{1}{2}
\end{align*}
Overall, we get that, 
\begin{align*}
\sqrt{\mathcal{E}_t(\hat{f})} \le 5\sqrt{\frac{B}{\beta}} \sqrt{2\mathcal{E}_t(f)}+ \frac{1}{2}\Phi_t(Q(t)) +\frac{1}{2}
\end{align*}
Then, we bound $\max_{1\leq \tau \leq t}\bar{L}_{\tau}$,
\begin{align*}
    \max_{1\leq\tau\leq t} \bar{L}_t \leq (2D)^{-1}\max_{1\leq\tau\leq t} (1 + \lambda_{\tau} e^{\lambda_\tau Q(\tau)}) \leq (2D)^{-1} + (2D)^{-1}\max_{1\leq\tau\leq t} \lambda_{\tau} e^{\lambda_\tau Q(\tau)} \leq (2D)^{-1} + (2D)^{-1} \lambda_{1} e^{\lambda_t Q(t)}
\end{align*}
Note that, $5\frac{B}{\beta}\lambda_1 \leq \frac{1}{4D^{-1}}$, so $(2D)^{-1}5\frac{B}{\beta} \lambda_{1} \leq \frac{1}{8}$.
Overall, we get,
\begin{align*}
    5\frac{B}{\beta}\max_{1\leq\tau\leq t} \bar{L}_t \leq 5\frac{B}{\beta}(2D)^{-1} + \frac{1}{8}(\Phi_t(Q(t))+1) 
\end{align*}
Substituting the upper-bound of $5\sqrt{\frac{B}{\beta}}\sqrt{\mathcal{E}_{t}(\hat{f})}$ and $5\frac{B}{\beta}\max_{1\leq \tau \leq t}\bar{L}_{\tau}$ in inequality \eqref{eq:opt-reg-decomp-1}: 
\begin{align*}
    \frac{3}{4}\Phi_{t+1}(Q(t+1))+\textrm{Regret}_t(x^*) \le 5\sqrt{\frac{B}{\beta}} \sqrt{2\mathcal{E}_t(f)}+ \frac{1}{2}\Phi_t(Q(t)) +\frac{1}{2} + 5\frac{B}{\beta}(2D)^{-1} + \frac{1}{8}(\Phi_t(Q(t))+1)
\end{align*}
Further, note that, $\Phi_t(Q(t)) \leq \Phi_{t+1}(Q(t+1))$,
\begin{align*}
    \frac{3}{4}\Phi_{t+1}(Q(t+1))+\textrm{Regret}_t(x^*) \le 5\sqrt{\frac{B}{\beta}} \sqrt{2\mathcal{E}_t(f)}+ \frac{5}{8} + 5\frac{B}{\beta}(2D)^{-1} + \frac{5}{8}\Phi_{t+1}(Q(t+1))
\end{align*}
\begin{align*}
    \frac{1}{8}\Phi_{t+1}(Q(t+1))+\textrm{Regret}_t(x^*) \le 5\sqrt{\frac{B}{\beta}} \sqrt{2\mathcal{E}_t(f)}+ \frac{5}{8} + 5\frac{B}{\beta}(2D)^{-1} 
\end{align*}
Dropping the non-negative $\Phi_{t}(Q(t))$ term gives the regret bound for the pre-processed functions:
\begin{align*}
    \textrm{Regret}_t(x^*) \le 5\sqrt{\frac{B}{\beta}} \sqrt{2\mathcal{E}_t(f)}+ \frac{5}{8} + 5\frac{B}{\beta}(2D)^{-1}
\end{align*}

For the CCV bound, use the trivial lower bound $\textrm{Regret}_t(x^*) \ge - Dt/ (2D) \cdot \alpha = -t / 2$. Substitute this into the inequality before dropping the $\Phi$ term:
\begin{align*}
\frac{1}{8}\Phi_{t+1}(Q(t+1)) \le 5\sqrt{\frac{B}{\beta}} \sqrt{2\mathcal{E}_t(f)}+ \frac{5}{8} + 5\frac{B}{\beta}(2D)^{-1} + \frac{t}{2}
\end{align*}
\begin{align*}
\Phi_{t+1}(Q(t+1)) \le 40\sqrt{\frac{B}{\beta}} \sqrt{2\mathcal{E}_t(f)}+ 10 + 20\frac{B}{\beta D} + 4t 
\end{align*}
Since $\Phi_{t}(x) = e^{\lambda_{t} x} - 1$:
\begin{align*}
e^{\lambda_{t+1} Q(t+1)} \le 40\sqrt{\frac{B}{\beta}} \sqrt{2\mathcal{E}_t(f)}+ 11 + 20\frac{B}{\beta D} + 4t 
\end{align*}
\begin{align*}
Q(t+1) \le \frac{1}{\lambda_{t+1}} \log(40\sqrt{\frac{B}{\beta}} \sqrt{2\mathcal{E}_t(f)}+ 11 + 20\frac{B}{\beta D} + 4t)
\end{align*}
\begin{align*}
    &Q(t+1) \\ \le & 20(\sqrt{\frac{B}{\beta}}+\frac{B}{\beta})\sqrt{\gamma_{t+1} + 1} \sqrt{\log(\gamma_{t+1} + 1)+1} (\log(\log(\gamma_{t+1}+1)+1)+1) \log(40\sqrt{\frac{B}{\beta}} \sqrt{2\mathcal{E}_t(f)}+ 11 + 20\frac{B}{\beta D} + 4t)
\end{align*}
where $\gamma_{t+1} = \mathcal{E}_t(\tilde{g})+D^{-2}$.

Overall, we get a bound of $O(\sqrt{\mathcal{E}_t(\tilde{f})})$ for regret and a bound of $\tilde{O}(\sqrt{\mathcal{E}_{t}(\tilde{g})})$ for the CCV.
\section{Numerical Simulations}
\label{expts}
\paragraph{Setup:} We consider a constrained version of the online shortest path problem \citep{hazan2022introduction}, where the length corresponds to the latency and the constraint is on the long-term cumulative bandwidth across the path. In this problem, on each round, the online algorithm first selects a route connecting a source $s$ to a destination $d$ on a graph $G(V,E)$. The latency and bandwidth of each edge vary across rounds, reflecting dynamic network conditions. The objective is to minimize the cumulative latency subject to a long-term lower bound on the cumulative bandwidth. 
% We take the problem of navigation through the roads in a city while trying to minimise the time delay and fuel consumption. The abstract version of this problem is the Online Shortest Paths on Graph problem as the locations can be modelled as vertices and the roads can be modelled as edges.
Formally, 
%we are given a directed graph $G = (V, E)$ and a source-destination pair $s, d \in V.$ Then at 
the following sequence of events takes place on the $t$\textsuperscript{th} round: 
\begin{enumerate}
    \item The algorithm first chooses a route (randomly or otherwise) $p_t \in P_{s,d}$, where $P_{s,d}$ is the set of all $s-d$ routes in the graph. 
    \item  A latency of $\tau_e(t)$ and a bandwidth of $l_e(t)$ is chosen by an adversary for each edge $e \in E$. 
    \item The algorithm incurs a latency cost of $\sum_{e \in p_t} \tau_e(t)$ and a bandwidth cost of $-\sum_{e \in p_t} l_e(t)$ on round $t.$
\end{enumerate}
%Our specific problem setting is similar except that we not only have costs $f_t$ but also constraints $g_t$. This corresponds to modelling the delay of a path as the cost and the estimated fuel consumption when travelling across it as the constraint.

%We associate each path $x$ with a vector $\{0,1\}^{|E|}$, where $x(i)$ represents the presence of the $i$th edge in the path. 
We represent each route $p$ by its corresponding $|E|$-dimensional binary incidence vector where $p_e=1$ if the edge $e \in E$ belongs to the route or $p_e=0$ otherwise. On each round, our online policy returns an element from the convex hull of $P_{s,d}$, also known as the unit flow polytope \citep{hazan2022introduction}. We use Dijkstra's algorithm for computing the weighted shortest route. It is well-known that any element in the unit flow polytope can be efficiently decomposed into a convex combination of at most $|E|$ number of $s-d$ routes using the flow decomposition lemma \citep[Lemma 2.20]{williamson2019network}. These convex combinations can be used to randomly select a single route on each round, incurring the same expected cost. The experiments are performed on a quad-core CPU with 8 GB RAM. The experimental setup mostly follows that in \citet{sarkar2025projectionfreealgorithmsonlineconvex}.

\paragraph{Dataset:}
For the experiments, we utilize a semi-synthetic dataset designed to simulate dynamic network conditions, constructed following methodologies established in prior work \citep{sarkar2025projectionfreealgorithmsonlineconvex}  for evaluating online network algorithms. The dataset originates from raw data collected via public probes on the RIPE Atlas global network measurement platform \citep{staff2015ripe}, specifically HTTPS measurements capturing bandwidth and latency between network nodes.The resulting graph consists of $n=191$ nodes and $m=1200$ edges, where each edge is characterized by latency and bandwidth attributes. The construction involves centering the graph around a hub node, adding connections based on the RIPE Atlas data, and then introducing additional random edges, with edge weights determined by latencies relative to the hub and minimum bandwidths between connected nodes.To mimic real-world fluctuations for the online setting, temporal variations are introduced through random scaling factors applied at each iteration: latency values are scaled randomly between 0.5x and 1.5x, and bandwidth values between 0.8x and 1.2x. This process generates the time-varying latency and bandwidth matrices used for evaluating algorithm performance over a horizon length of $T=1600$ iterations. This dataset provides a realistic and challenging testbed for constrained online optimization algorithms.

\begin{figure}[ht]
    \centering
    \begin{minipage}[b]{0.48\textwidth}
        \centering
        \includegraphics[width=\linewidth]{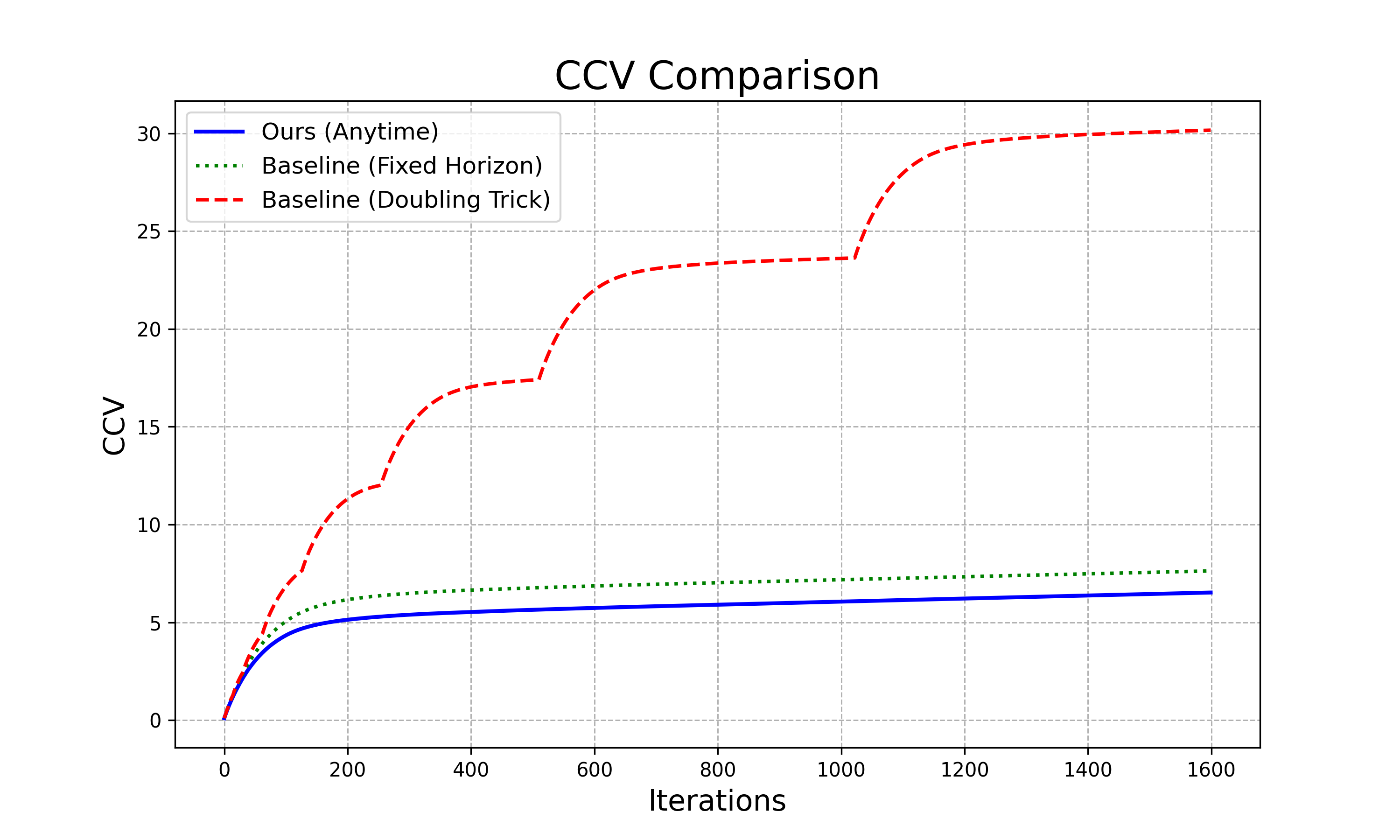}
        \caption{CCV comparison between our policy, Algorithm 1 of \citet{sinha2024optimal} and its doubling trick implementation.}
        \label{fig:ccv_comparison}
    \end{minipage}
    \hfill
    \begin{minipage}[b]{0.48\textwidth}
        \centering
        \includegraphics[width=\linewidth]{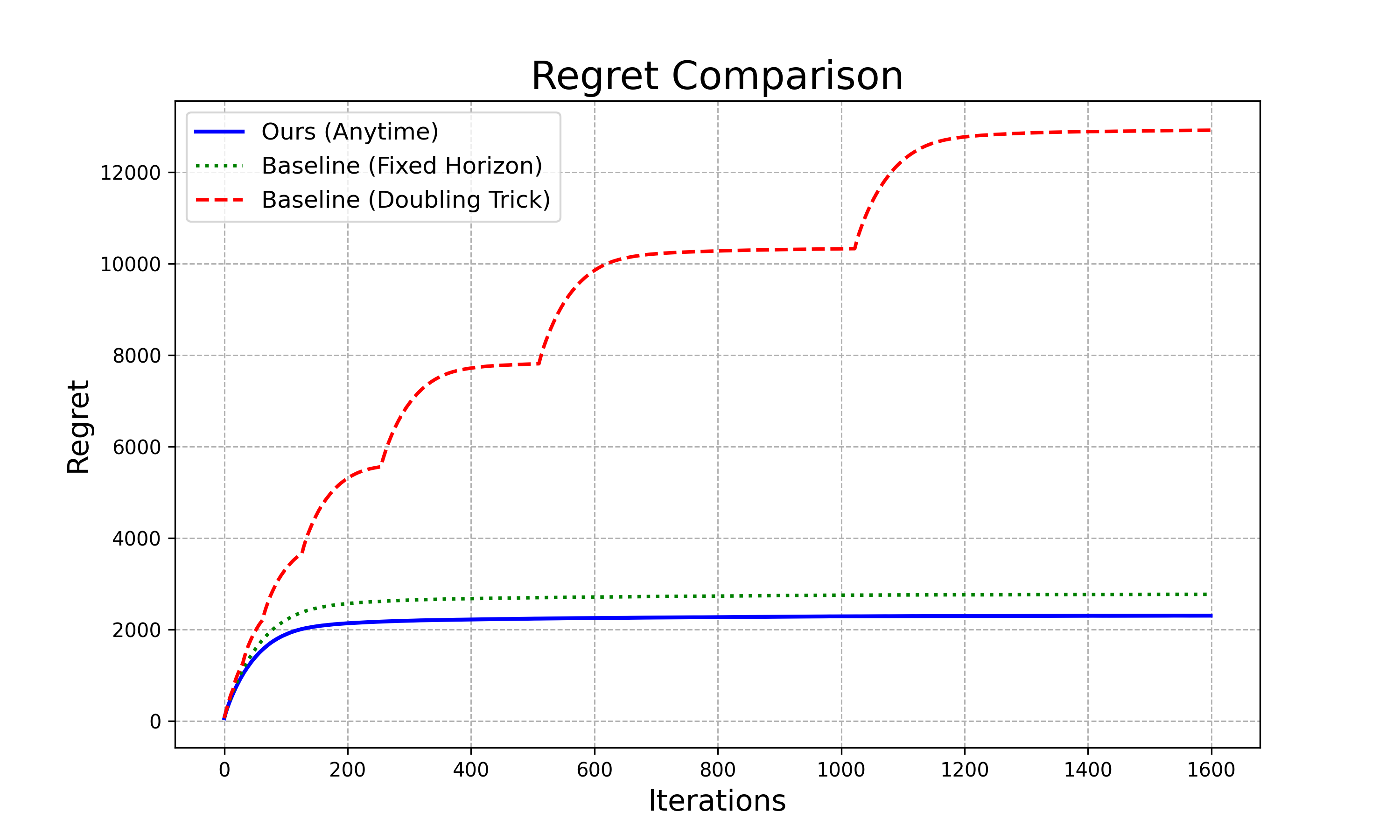}
        \caption{Regret comparison between our policy, Algorithm 1 of \citet{sinha2024optimal} and its doubling trick implementation..}
        \label{fig:regret_comparison}
    \end{minipage}
    %\caption{Performance of our policy compared to baselines.}
    \label{fig:combined}
\end{figure}

%Note that our algorithms do not return single paths; instead, they return a convex combination of paths. This can also be interpreted as a distribution over paths. Thus, our decision set is the convex hull of the set of all $s-t$ paths. Notice that each element of this decision set can also be considered a unit flow. Thus this decision set is essentially the flow polytope which is defined by $O(|E|)$ linear equations : those for ensuring positivity of flow and those for conservation of in-flow and out-flow for every vertex other than source/sink (which will have unit inflow/outflow respectively).

%It is also possible to decompose the unit flow to a distribution of at most $|E|$ paths in polynomial time using the flow decomposition algorithm.
\paragraph{Results:} The empirical results presented in Figure \ref{fig:ccv_comparison} and Figure \ref{fig:regret_comparison} provide strong validation for our theoretical contributions. The plots compare the performance of our proposed anytime algorithm against two key baselines: the standard fixed-horizon algorithm from Sinha and Vaze (2024), which assumes prior knowledge of the total horizon $T$, and its adaptation to an unknown horizon setting via the standard doubling trick.

In summary, the simulations reveal a key practical advantage of our approach, demonstrating that our anytime algorithm outperforms both the fixed-horizon and doubling trick baselines. 

The doubling trick based method performs the worst. This is due to the restarting technique that they employ where an the algorithm has to discard the information it has upto that point and start from scratch. Our observations match those in \cite{besson2018doubling} where they evaluated several versions of the doubling trick and concluded that performance is substantially worse than that of an intrinsically anytime algorithm, even one with less favorable theoretical guarantees. Moreover, this performance gap becomes larger each time the algorithm is restarted.

While the fixed-horizon method is optimal in a worst-case theoretical sense, its use of a single, conservative trade-off parameter hinders its practical performance. Our algorithm, with its time-varying parameter, is more adaptive; it aggressively penalizes violations in the early stages, quickly learning the constraints and achieving a superior performance trajectory. This confirms that our method is not just a tool for handling an unknown horizon, but a fundamentally more effective strategy that also avoids the instability and performance loss inherent in the doubling trick.
\end{document}